\documentclass{article}
\usepackage[nonatbib, preprint, logo]{cig}
\usepackage[utf8]{inputenc} % allow utf-8 input
\usepackage[T1]{fontenc}    % use 8-bit T1 fonts
\usepackage{hyperref}
\usepackage{xcolor}
\usepackage{booktabs}
\usepackage{arydshln}
\usepackage{balance}
\usepackage{multirow}
\usepackage{multicol}
\usepackage{algorithm}
\usepackage{algpseudocode}
\usepackage{enumitem}
\definecolor{linkblue}{RGB}{0,0,128} 
\hypersetup{
    colorlinks=true,
    linkcolor=black,   % 
    citecolor=linkblue,   % citations
    urlcolor=linkblue,    % URLs
    filecolor=linkblue
}

\usepackage{wrapfig}
\usepackage{soul}
\definecolor{mellowgreen}{RGB}{225, 242, 225}
\definecolor{mellowblue}{RGB}{225, 255, 248}
\definecolor{mellowyellow}{RGB}{255, 248, 210}
\definecolor{mellowred}{RGB}{250, 225, 225}
\newcommand{\bluehl}[1]{{\sethlcolor{mellowblue}\hl{#1}}}
\newcommand{\redhl}[1]{{\sethlcolor{mellowred}\hl{#1}}}
\algtext*{EndFor}
\usepackage{amsmath,amsfonts,bm, xcolor,mathtools}
\usepackage{amsthm}
\usepackage{soul}
\DeclareMathOperator{\tr}{tr}

\def\eqref#1{equation~\ref{#1}}
\def\1{\bm{1}}

\def\zerobm{{\bm{0}}}

\def\vbm{{\bm{v}}}
\def\ebm{{\bm{e}}}

\def\\px{{\bm{x}}}
\def\ubm{{\bm{u}}}
\def\zbm{{\bm{z}}}
\def\ybm{{\bm{y}}}

\def\zbm{{\bm{z}}}
\def\sbm{{\bm{s}}}

\def\nbm{{\bm{n}}}
\def\mbm{{\bm{m}}}

\def\xbm{{\bm{x}}}

\def\kbm{{\bm{k}}}
\def\epsilonbm{{\bm{\epsilon}}}

\def\Mbm{{\bm{M}}}
\def\Ubm{{\bm{U}}}
\def\Vbm{{\bm{V}}}

\def\Abm{{\bm{A}}}
\def\Sbm{{\bm{S}}}

\def\Hbm{{\mathbf{H}}}

\def\Pbm{{\bm{P}}}
\def\Fbm{{\bm{F}}}
\def\Ibm{{\bm{I}}}

\def\\px{{\bm{X}}}

\def\Vbm{{\bm{V}}}
\def\Wbm{{\bm{W}}}

\def\Sigmabm{{\bm{\Sigma}}}

\def\Ncal{{\mathcal{N}}}

\def\Lcal{{\mathcal{L}}}

\def\Ucal{{\mathcal{U}}}

\def\xbmhat{{\widehat{\bm{x}}}}

\theoremstyle{plain}

\newtheorem{proposition}{Proposition}
\newtheorem{lemma}{Lemma}

\theoremstyle{definition}

\theoremstyle{remark}

\def\thetabm{{\bm{\theta}}}

\DeclareMathAlphabet{\mathsfit}{\encodingdefault}{\sfdefault}{m}{sl}
\SetMathAlphabet{\mathsfit}{bold}{\encodingdefault}{\sfdefault}{bx}{n}

\newcommand{\xworkprior}{\epsbm}
\newcommand{\E}{\mathbb{E}}

\newcommand{\R}{\mathbb{R}}

\usepackage{booktabs}
\usepackage{multirow}
\usepackage{graphicx}
\usepackage{enumitem}
\usepackage{algorithm}
\usepackage{algpseudocode}
\IfFileExists{xspace.sty}{\usepackage{xspace}}{\newcommand{\xspace}{}}

\definecolor{lightgreen}{rgb}{.9,1,.9}
\definecolor{trolleygrey}{rgb}{0.5, 0.5, 0.5}
\definecolor{BrickRed}{rgb}{0.6,0,0}
\definecolor{RoyalBlue}{rgb}{0,0,0.8}
\definecolor{Tdgreen}{rgb}{0,0.4,0.7}
\definecolor{pinegreen}{rgb}{0.0, 0.47, 0.44}
\definecolor{cornellred}{rgb}{0.7, 0.11, 0.11}
\definecolor{cadmiumgreen}{rgb}{0.0, 0.42, 0.24}
\definecolor{spirodiscoball}{rgb}{0.06, 0.75, 0.99}
\definecolor{mylightblue}{rgb}{0.85, 0.90, 0.94}
\definecolor{maroon}{cmyk}{0,0.87,0.68,0.32}
\definecolor{mydarkblue}{RGB}{0,0,128}  % navy‑ish

\def\defn{\,:=\,}
\def\Kbm{{\bm{K}}}

\def\epsbm{{\bm{\epsilon}}}

\DeclareMathOperator{\Cov}{Cov}
\DeclareMathOperator{\Var}{Var}
\DeclareMathOperator{\diag}{diag}
\DeclareMathOperator{\rank}{rank}
\DeclareMathOperator*{\nullsp}{null}

\DeclareMathOperator{\range}{range}

\newcommand{\sg}[1]{\mathrm{sg}\!\left[#1\right]}
\newcommand{\pinv}[1]{#1^{\dagger}}

\newcommand{\norm}[1]{\left\lVert #1 \right\rVert}
\newcommand{\abs}[1]{\left\lvert #1 \right\rvert}

\newcommand{\given}{\,\vert\,}

\newcommand{\myv}{\mbm_{\ybm}}                % posterior mean of the source
\newcommand{\sn}{\sigma_{\mathrm{n}}}         % measurement noise level
\newcommand{\utheta}[2]{\ubm^{\thetabm}_{#1,#2}}

\newif\ifshowgaps
\showgapstrue

  {\ifshowgaps
     \par\medskip\noindent
     \begin{minipage}{\linewidth}%
       \color{cornellred}%
       \hrule height 0.6pt\relax
       \par\vspace{3pt}\noindent
       \textbf{MISSING --- #1}\par\vspace{2pt}\noindent
       \small\itshape
   \else
     \setbox0\vbox\bgroup
   \fi}%
  {\ifshowgaps
       \par\vspace{3pt}\hrule height 0.6pt\relax
     \end{minipage}\par\medskip
   \else
     \egroup
   \fi}

\newcommand{\figorbox}[3]{% {file}{width}{caption-if-missing}
  \IfFileExists{#1}{\includegraphics[width=#2]{#1}}%
  {\fbox{\begin{minipage}[c][0.22\textheight][c]{#2}\centering
     \color{trolleygrey}\small\texttt{\detokenize{#1}}\\[4pt]\itshape #3
   \end{minipage}}}}

\def\dd{{\mathrm{d}}}

\usepackage{xspace}

\newcommand{\methodabbrev}{\ensuremath{\text{SNaP}}\xspace}
\title{\methodabbrev: One-Step Posterior Sampling for\newline Noisy Inverse Problems}

\vspace{5em}
\author{%
\normalsize Shirin Shoushtari$^{1*}$ \quad  
Edward P.~Chandler$^{2*}$ \quad Xiao Shi$^2$ \quad Ulugbek S.~Kamilov$^2$\\[0.7em]
\small \textnormal{$^1$WashU, $^2$UW-Madison}\\[0.5em]
\footnotesize \texttt{s.shirin@wustl.edu, epchandler@wisc.edu, xiao.shi@wisc.edu, kamilov@wisc.edu}\\
\footnotesize \texttt{$^*$Equal contribution}}

\begin{document}
\maketitle

\begin{abstract}
Diffusion and flow-matching models can produce high-quality posterior samples for inverse problems, but typically require tens to thousands of network evaluations per draw. MeanFlow enables one-step generation, yet applying it to inverse problems leaves no intermediate steps at which to enforce measurement consistency. 
We introduce \methodabbrev, a \emph{one-step} MeanFlow posterior sampler for linear inverse problems with Gaussian noise.
Its central innovation is a measurement-adapted source: a Gaussian distribution whose mean and anisotropic covariance are determined by the measurement operator, observation, and noise level.
The source anchors well-measured directions while preserving variation where the measurements are weak or uninformative. 
We show that the exact conditional flow transports this source to the true posterior.
Across natural-image restoration and multi-coil MRI, \methodabbrev produces diverse, high-quality samples with one network evaluation per draw, 30 to 2250$\times$ faster than iterative samplers.
\end{abstract}

\section{Introduction}
\label{sec:intro}

Many imaging tasks require recovering an unknown image $\xbm_1$ from noisy linear measurements $\ybm=\Abm\xbm_1+\nbm$, where $\Abm$ models the imaging system and $\nbm\sim\Ncal(\zerobm,\sn^2\Ibm)$ denotes additive Gaussian noise. 
This formulation encompasses magnetic resonance imaging (MRI), computed tomography (CT), microscopy, astronomical imaging, and remote sensing~\cite{bertero2021introduction, mccann2017convolutional, ongie2020deep}.
Because $\Abm$ is often ill-conditioned or rank-deficient, the measurements do not uniquely determine $\xbm_1$. Therefore, we need prior information on images for a reliable recovery. 
Classical methods impose handcrafted regularizers such as total variation \cite{rudin1992nonlinear}. 
More recently, deep generative models, including diffusion models \cite{ho2020denoising,song2020score} and flow-matching models \cite{lipman2023flow,liu2023flow,albergo2025stochastic}, have learned rich image distributions that can serve as priors for inverse problems \cite{daras2024survey}.

Most generative inverse solvers use an unconditional generative model as an image prior and impose measurement consistency repeatedly during an iterative sampling trajectory. One family modifies the generative dynamics using data-fidelity gradients or pseudoinverse corrections \cite{chung2023diffusion,song2023pseudoinverseguided,kawar2022denoising,wang2023zeroshot,pokle2024training,benhamu2024dflow}. Another alternates generative updates with explicit data-consistency steps, following the plug-and-play paradigm \cite{zhu2023diffpir,chung2024decomposed,zhang2025improving,martin2025pnpflow,pourya2026flower}. Although these corrections improve agreement with the measurement $\ybm$, they require many sequential network evaluations and therefore make posterior sampling slow. 
Depending on the method, generating a single sample may require tens to thousands of network evaluations (NFEs).
In our CelebA Gaussian-deblurring experiment, for example, DPS~\cite{chung2023diffusion} requires $1000$ NFEs and $27$ seconds to produce one ($128 \times 128$) reconstruction (Figure~\ref{fig:teaser}). 

Consistency models \cite{song2023consistency,boffi2026build}, Shortcut Models \cite{frans2025shortcut}, and MeanFlow\cite{geng2025meanflow} enable generation with one network evaluation, avoiding the cost of integrating an iterative sampling trajectory. This efficiency is particularly attractive for posterior sampling, where generating a single reconstruction can require tens to thousands of network evaluations. However, collapsing the trajectory into one step also eliminates the intermediate states for enforcing measurement consistency. A direct extension is to condition MeanFlow on the measurement while retaining the standard isotropic source. Recent studies of one-step generation have reported mean-seeking bias in MeanFlow and averaging-induced diversity degradation in class-conditional flow models~\cite{he2026stabilizing,lin2026subflow}. We observe the same failure mode in inverse problems: despite conditioning on the measurement, a MeanFlow with an isotropic source produces nearly identical reconstructions across draws, so averaging offers no improvement (Table~\ref{tab:ablsource}). This motivates shaping the one-step transport to reflect the geometry of the forward model.

\begin{figure}[t!]
\centering
\includegraphics[width=\textwidth]{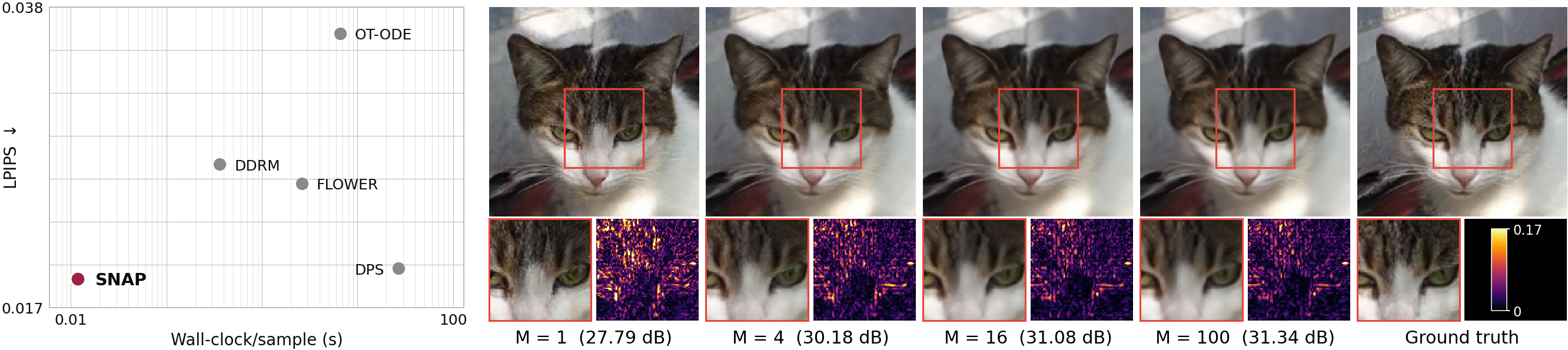}
\caption{\emph{Left:} LPIPS versus wall-clock time per sample for CelebA Gaussian deblurring. \methodabbrev achieves lowest LPIPS at a fraction of the cost of iterative samplers.
\emph{Right:} $4\times$ super-resolution on AFHQ-Cat. Columns show one \methodabbrev draw, averages of $M=4$, $16$, and $100$ draws, and the ground truth; zoomed crops, error maps, and PSNR are shown below. Each draw requires one network evaluation. Averaging reduces pixel error and improves PSNR while smoothing fine detail.}
\label{fig:teaser}
\end{figure}

We propose \methodabbrev, a one-step sampler for \emph{noisy} linear inverse problems. Its key innovation is a \emph{measurement-adapted} Gaussian source that builds the forward model and noise level into the starting point of the transport. The source anchors well-measured directions while retaining stochastic variation in weakly measured and null-space directions. Samples from this source can be drawn efficiently, and a single network evaluation maps each draw to a reconstruction.

Our contributions are: \textbf{(a)} We extend one-step MeanFlow sampling to noisy linear inverse problems by constructing a measurement-adapted transport aligned with the geometry of the likelihood. \textbf{(b)} We derive the source center by minimizing its expected squared displacement to posterior samples among linear estimators. The resulting construction uses a Tikhonov estimate to anchor the transport. \textbf{(c)} Across natural-image restoration and multi-coil MRI, \methodabbrev produces diverse, high-quality reconstructions with one network evaluation per draw. Its samples achieve near-unit calibration ratios on the tested CelebA tasks, while sampling $30$ to $2250\times$ faster than iterative posterior samplers.

\section{Background}
\label{sec:background}
  
\textbf{Inverse Problems.} We consider a linear measurement operator $\Abm\in\R^{m\times n}$ with additive white Gaussian noise,
\begin{equation}
\ybm \;=\; \Abm\xbm_1 + \nbm, \qquad \nbm\sim\Ncal(\zerobm,\sn^{2}\Ibm).
\label{eq:fwd}
\end{equation}
The goal is to infer the unknown image $\xbm_1\in\R^n$ from the measurement $\ybm\in\R^m$; in posterior sampling, this means drawing samples from $p(\xbm_1\mid\ybm)$. 
We write the singular value decomposition of $\Abm$ as $\Abm=\Ubm\Sigmabm\Vbm^{\top}$ and let $\rho\defn\rank(\Abm)$. 
We order the singular values so that $s_1\ge\cdots\ge s_\rho>0$ and set $s_i=0$ for $\rho<i\le n$.
The corresponding right singular vectors are the columns of $\Vbm$.
We denote the pseudoinverse  by $\pinv{\Abm}$ and the orthogonal projection onto $\nullsp(\Abm)$ by
$\Pbm\defn\Ibm-\pinv{\Abm}\Abm$.

\textbf{Flow Matching.} Flow matching \cite{lipman2023flow,liu2023flow,albergo2025stochastic} learns a velocity field that transports a source distribution $p_0$ to a data distribution $p_1$. We index time so that $t=0$ corresponds to the source and $t=1$ to the data. The flow $\psi_t$ is defined by the ODE $\tfrac{\dd}{\dd t}\psi_t(\xbm)=\vbm_t(\psi_t(\xbm))$ for $t\in(0,1)$, and it maps samples from $p_0$ to samples from $p_1$. 
For the exact marginal velocity field, the terminal map satisfies $(\psi_1)_{\#}p_0=p_1$.
This marginal velocity is generally not available in closed form, so flow matching instead regresses a sample-conditional velocity.
Given a coupling of endpoints $(\xbm_0,\xbm_1)$ with $\xbm_0\sim p_0$ and $\xbm_1\sim p_1$, the straight-line path
\begin{equation}
\zbm_t=(1-t)\,\xbm_0+t\,\xbm_1,
\qquad
\vbm=\xbm_1-\xbm_0,
\label{eq:path}
\end{equation}
has constant conditional velocity $\vbm$.
A network $\vbm^{\thetabm}$ is trained by minimizing $\E\,\|\vbm^{\thetabm}(\zbm_t,t)-(\xbm_1-\xbm_0)\|_2^{2}$. The expected parameter gradient of this conditional regression objective coincides with that of the marginal flow-matching objective \cite{lipman2023flow}.
The source is commonly chosen as $\Ncal(\zerobm,\Ibm)$ because it is easy to sample; its random initial conditions provide the variability propagated by the deterministic ODE.

\textbf{Mean Flows.} Sampling a flow-matching model requires integrating the velocity $\vbm$, which takes many network evaluations.
MeanFlow \cite{geng2025meanflow,boffi2026build} avoids this integration by modeling the average velocity $\ubm$ over an interval $[r,t]$, defined as
\begin{equation}
\ubm(\zbm_r,r,t) \;\defn\; \frac{1}{t-r}\int_r^{t}\vbm(\zbm_s,s)\,\dd s.
\label{eq:avgvel}
\end{equation}
Once $\ubm$ is learned, the entire flow reduces to a single evaluation, $\zbm_1=\zbm_0+\ubm(\zbm_0,0,1)$.
Learning $\ubm$ directly from \eqref{eq:avgvel} is intractable, since it requires the integral of $\vbm$.
Instead, rewriting \eqref{eq:avgvel} as $(t-r)\,\ubm=\int_r^t\vbm\,\dd s$ and differentiating with respect to $r$, with $t$ fixed, gives the MeanFlow identity
\begin{equation}
\ubm(\zbm_r,r,t) \;=\; \vbm(\zbm_r,r) \;+\; (t-r)\,\frac{\dd}{\dd r}\ubm(\zbm_r,r,t).
\label{eq:mfidentity}
\end{equation}
This differential identity provides the training relation used by MeanFlow.
It also connects MeanFlow to consistency and shortcut models, which learn finite-time flow maps through related consistency relations \cite{song2023consistency,boffi2026build,frans2025shortcut}.

\textbf{Measurement-dependent sources.} Standard MeanFlow formulations use a fixed, measurement-independent source. Applying MeanFlow to inverse problems therefore requires incorporating the measurement model into the one-step transport.
NullFlow \cite{shi2026nullflow} is a one-step posterior sampler for noiseless linear inverse problems that uses a measurement-dependent source:
\begin{equation}
\xbm_0 \;=\; \pinv{\Abm}\ybm + \Pbm\xworkprior,
\qquad \xworkprior\sim\Ncal(\zerobm,\Ibm).
\label{eq:nf-source}
\end{equation}
For a noiseless measurement, the source satisfies $\Abm\xbm_0=\ybm$, and restricting the velocity to $\nullsp(\Abm)$ keeps the entire trajectory in the measurement-consistent affine set
$\{\xbm:\Abm\xbm=\ybm\}$;
however, with noisy measurements, the posterior is no longer supported on this set. Moreover, the anchor \(\pinv{\Abm}\ybm\) contains the row-space noise term \(\pinv{\Abm}\nbm\), which a velocity restricted to \(\nullsp(\Abm)\) cannot alter.
\section{Method}
\label{sec:method}

\methodabbrev builds the geometry of the forward model into a one-step posterior transport. Its measurement-dependent source is designed to keep the path to the posterior short while preserving the randomness needed for sampling. It concentrates draws in well-measured directions and retains variation in weakly measured and null-space directions. A conditional MeanFlow then learns to transport these draws toward $p(\xbm_1\given\ybm)$, producing an approximate posterior sample from one source draw and one network evaluation. We now derive the source, characterize the exact transport, and describe how the model is trained and sampled.

%shapes the one-step transport according to the geometry of the forward model. We design a measurement-dependent source with two goals: (i) reducing the expected squared displacement to posterior samples and (ii) preserving the stochastic variation required for posterior sampling. The source assigns lower variance to directions that are strongly constrained by the measurement and higher variance to weakly measured and null-space directions. We then train a conditional MeanFlow to approximate the resulting transport to $p(\xbm_1\given\ybm)$. At inference, one source draw and one network evaluation produce an approximate posterior sample. We next derive the source construction, characterize the target of the exact transport, and describe training and sampling.

Along the straight-line path between a source sample $\xbm_0$ and a posterior sample $\xbm_1$, the sample-conditional velocity is $\vbm=\xbm_1-\xbm_0$, and therefore
$\|\vbm\|_2^2=\|\xbm_1-\xbm_0\|_2^2$ \cite{liu2023flow}. We keep the path and the conditional-independence structure of the endpoint coupling fixed and use the source distribution as our design variable. This differs from coupling-based methods that modify the endpoint pairing \cite{pooladian2023multisample,tong2024improving,albergostochastic}. Unlike approaches that prescribe or learn data-dependent sources \cite{lee2022priorgrad,mammadov2026variational}, we derive the source from the measurement model using an explicit design criterion.

Although the posterior is unavailable in closed form, the likelihood provides a tractable geometry for images compatible with the measurement. Under additive Gaussian noise,
\[
p(\ybm\given\xbm)\propto
\exp\left(-\frac{\|\Abm\xbm-\ybm\|_2^2}{2\sn^2}\right),
\]
and is maximized on the least-squares solution set
\[
\mathcal S_{\ybm}\defn\arg\min_{\xbm}\|\Abm\xbm-\ybm\|_2^2 .
\]
Its superlevel sets are convex tubes around $\mathcal S_{\ybm}$. Along a measured right singular vector $\vbm_i$, their width scales as $\sn/s_i$: well-measured directions are tightly constrained, weakly measured directions are less constrained, and null-space directions are unconstrained (Fig.~\ref{fig:tube}a). The noise level sets the scale because $\Abm\xbm_1-\ybm=-\nbm$. Since these tubes are convex, a straight path between two points in the same tube stays inside it. This geometry motivates a source with little variation in well-measured directions and greater variation in weakly measured and null-space directions (Fig.~\ref{fig:tube}c), in contrast to an isotropic source (Fig.~\ref{fig:tube}b). We construct this source from the observed $\ybm$, known $\Abm$ and $\sn$, and auxiliary randomness. We consider sources with an explicit, linear measurement-dependent center:
\begin{equation}\label{eq:source-family}
\xbm_0=\bm K\ybm+\bm\xi,
\end{equation}
where $\bm K$ is fixed for a given measurement model and $\bm\xi$ has mean $\zerobm$ and covariance $\bm\Sigma$, independently of $(\xbm_1,\nbm)$. This choice admits an exact analysis and avoids training a separate source model \cite{mammadov2026variational}. The family contains the standard isotropic source when $\bm K=\zerobm$ and $\bm\Sigma=\Ibm$, and sources centered at the least-squares solution when $\bm K=\pinv{\Abm}$.

Because the one-step map is deterministic given $(\xbm_0,\ybm)$, all sampling randomness originates from the source.
For the family in \eqref{eq:source-family}, its expected squared displacement decomposes as
\begin{equation}
\E\!\left[\|\xbm_1-\xbm_0\|_2^2\right]
=
\tr\!\left(\bm R(\bm K)\right)
+
\tr(\bm\Sigma),
\qquad
\bm R(\bm K)
\defn
\E\!\left[
(\xbm_1-\bm K\ybm)
(\xbm_1-\bm K\ybm)^\top
\right].
\label{eq:displacement-decomposition}
\end{equation}
For fixed $\bm\Sigma$, minimizing the displacement reduces to minimizing the center error $\tr(\bm R(\bm K))$ over $\bm K$.
Minimizing over $\bm\Sigma$ would instead give $\bm\Sigma=\zerobm$ and eliminate source variability, so we optimize the center and choose the spread separately.

\begin{figure}[t]
\centering
\includegraphics[width=\textwidth]{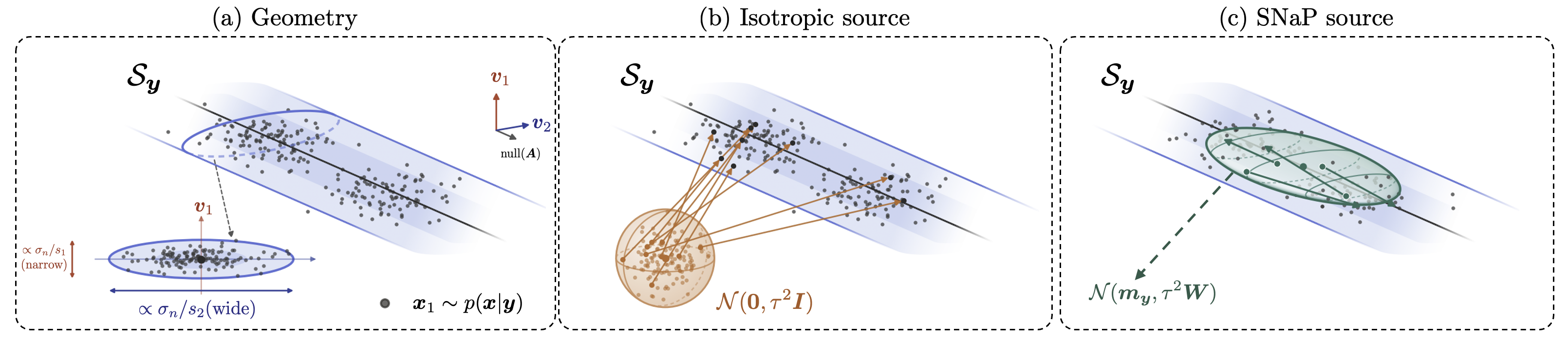}
\vspace{-1em}
\caption{Geometry-aware one-step transport (schematic).
\textbf{Left:} Likelihood superlevel sets form tubes around $\mathcal S_{\ybm}$: narrow in well-measured directions, wider in weakly measured directions, and unbounded along $\nullsp(\Abm)$. Gray points represent posterior samples.
\textbf{Middle:} An isotropic source spreads samples across directions regardless of measurement strength, leading to large source-to-target displacements.
\textbf{Right:} The \methodabbrev source concentrates samples in well-measured directions while retaining variation in weakly measured and null-space directions, shortening displacements.}
\vspace{-1.5em}
\label{fig:tube}
\end{figure}

\begin{proposition}
\label{prop:optimal-source}
Let $\xbm_1$ have mean $\zerobm$ and covariance $\bm C$, not necessarily Gaussian, and let $\ybm=\Abm\xbm_1+\nbm$ with $\nbm\sim\Ncal(\zerobm,\sn^2\Ibm)$  and $\sn>0$.
For any fixed $\bm\Sigma\succeq\zerobm$, the expected squared displacement of the source family in \eqref{eq:source-family} is uniquely minimized over $\bm K$ by
\begin{equation}
\bm K_{\bm C}
\defn
\bm C\Abm^\top
\bigl(\Abm\bm C\Abm^\top+\sn^2\Ibm\bigr)^{-1}.
\label{eq:optimal-source}
\end{equation}
$\bm K_{\bm C}\ybm$ is thus the linear minimum-mean-squared-error estimator of $\xbm_1$ from $\ybm$,
with error covariance
\[
\bm\Sigma_{\bm C}
\defn
\bm R(\bm K_{\bm C})
=
\bm C-\bm K_{\bm C}\Abm\bm C.
\]

For the isotropic working covariance $\bm C=\tau^2\Ibm$, with $\tau>0$, the optimal center is the Tikhonov estimate 
$\myv
=
\Abm^\top
\bigl(\Abm\Abm^\top+\lambda\Ibm\bigr)^{-1}\ybm$, with $\lambda=\sn^2/\tau^2$,
and its error covariance is $\tau^2\Wbm$, where
$
\Wbm
=
\Ibm-
\Abm^\top
\bigl(\Abm\Abm^\top+\lambda\Ibm\bigr)^{-1}\Abm$.
\end{proposition}

The proof is given in App.~\ref{app:proof-optimal-source}. Minimizing the displacement over $\bm\Sigma$ would give $\bm\Sigma=\zerobm$ and collapse the source.
For a chosen working covariance $\bm C$, we instead set $\bm\Sigma=\bm\Sigma_{\bm C}$, matching the error covariance of the optimal linear estimator. Under the corresponding Gaussian working model, this is also the posterior covariance.

Because the true image covariance is unknown, we instantiate Proposition~\ref{prop:optimal-source} with the isotropic working covariance $\bm C=\tau^2\Ibm$, where $\tau$ controls the source scale. Under this model, $\myv$ is the displacement-minimizing linear center and $\tau^2\Wbm$ is its error covariance. The resulting Tikhonov center regularizes weakly measured directions and avoids amplifying their measurement noise, unlike the least-squares estimate
$\pinv{\Abm}\ybm$. The working model is used only to construct the source; it does not assume that the true image distribution is isotropic or Gaussian. Using this mean and covariance, we define
\begin{equation}
q_\tau(\xbm\given\ybm)
=
\Ncal(\xbm;\myv,\tau^2\Wbm).
\label{eq:working-posterior}
\end{equation}

In the right singular basis of $\Abm$,
\begin{equation}
\Wbm
=
\Vbm\diag(w_i)\Vbm^\top,
\qquad
w_i
=
\frac{\sn^2}{\sn^2+\tau^2s_i^2}.
\label{eq:source-spectrum}
\end{equation}
The source variance along $\vbm_i$ is therefore $\tau^2w_i$.
It is small in well-measured directions, approaches $\sn^2/s_i^2$ when $\tau^2s_i^2\gg\sn^2$, and equals $\tau^2$ on $\nullsp(\Abm)$.
Thus, the source suppresses variation in well-measured directions while retaining it where the measurement is weak or uninformative.
As $\sn\to0$, the isotropic working source converges to $q_\tau(\cdot\given\ybm) \to \Ncal(\pinv{\Abm}\ybm,\tau^2\Pbm)$,
which recovers the NullFlow source when $\tau=1$.

Replacing the isotropic source by the working posterior changes the transport, but not its target.
\begin{proposition}
\label{prop:invariance}
Let $\sn>0$, and let $\xbm_0\sim q_\tau(\cdot\given\ybm)$ and $\xbm_1\sim p(\cdot\given\ybm)$ be conditionally independent given $\ybm$ and joined by the straight-line path in \eqref{eq:path}.
Assume that $p(\cdot\given\ybm)$ has finite second moments, and that the flow of the marginal velocity field has a terminal limit: writing $\psi_{0\to t}$ for its flow map from time $0$ to time $t$, the limit $\psi_{0\to1}\defn\lim_{t\to1}\psi_{0\to t}$ exists $q_\tau(\cdot\given\ybm)$-almost everywhere.
Then the one-step map
\begin{equation}
T_{\ybm}(\xbm)\defn\xbm+\ubm(\xbm,0,1\given\ybm)
\label{eq:onestep-map}
\end{equation}
satisfies
\begin{equation}
(T_{\ybm})_{\#}\,q_{\tau}(\cdot\given\ybm)=p(\cdot\given\ybm)
\qquad\text{for every }\tau>0.
\label{eq:invariance}
\end{equation}
\end{proposition}
Here $\ubm$ is the exact average velocity of \eqref{eq:avgvel}; the trained network $\ubm^{\thetabm}$ approximates it, so the guarantee holds for \methodabbrev samples only up to this approximation error.
The exact-map result identifies the posterior target of SNaP; the effect
of the source choice on the learned one-step approximation is evaluated
through the controlled comparison with isotropic sources in Table~\ref{tab:ablsource}.
% We have some results that compare transport on sampels for our source vs. Gaussian, Maybe we should add those (I removed those when removing the preivous propostions on expected displacement and energy of intermediate samples) 

\subsection{Training and sampling}
\label{sec:training}
A direct draw from $q_\tau(\cdot\given\ybm)$ would require applying the matrix square root $\Wbm^{1/2}$, which is generally impractical for large imaging operators.
We instead use perturb-and-solve
\cite{papandreou2010gaussian,bardsley2014tro}:
\begin{equation}
\xbm_0=
\epsilonbm
+
\Abm^\top
\bigl(\Abm\Abm^\top+\lambda\Ibm\bigr)^{-1}
\bigl(\ybm-\Abm\epsilonbm-\nbm'\bigr),\quad
\epsilonbm\sim
\Ncal(\zerobm,\tau^2\Ibm),
\quad
\nbm'
\sim
\Ncal(\zerobm,\sn^2\Ibm).
\label{eq:rto}
\end{equation}
Here, $\epsilonbm$ and $\nbm'$ are independent of each other and of $(\xbm_1,\nbm)$.
Lemma~\ref{lem:rto} shows that
$\xbm_0\given\ybm\sim q_\tau(\cdot\given\ybm)$ exactly and that
$\xbm_0$ and $\xbm_1$ are conditionally independent given $\ybm$.
Each source draw requires one solve with
$\Abm\Abm^\top+\lambda\Ibm$; operator-specific implementations are given in App.~\ref{app:sourcesolve}.

\textbf{Training.}
The network parameterizes the conditional average velocity
\[
\utheta{r}{t}
=
\ubm^{\thetabm}(\zbm_r,r,t\given\ybm,\sn).
\]
Because $\ybm$ and $\sn$ remain fixed along the path, the conditional MeanFlow identity gives the target
\begin{equation}
\ubm_{\mathrm{tgt}}
=
\vbm
+
(t-r)\,
\sg{
\bigl(\partial_{\zbm}\utheta{r}{t}\bigr)\vbm
+
\partial_r\utheta{r}{t}
},
\qquad
\vbm=\xbm_1-\xbm_0.
\label{eq:mftarget}
\end{equation}
The derivative term is evaluated using one Jacobian--vector product, and $\sg{\cdot}$ prevents gradients from propagating through the target.
We minimize
\begin{equation}
\Lcal(\thetabm)
=
\E\!\left[
\sg{\omega(\ell)}\,\ell
\right],
\qquad
\ell
\defn
\frac{1}{n}
\left\|
\utheta{r}{t}-\ubm_{\mathrm{tgt}}
\right\|_2^2,
\label{eq:spring-loss}
\end{equation}
using the adaptive weight
$\omega(\ell)=(\ell+c)^{-p}$
from \cite{geng2025meanflow}.
Implementation details and Algorithm~\ref{alg:train} are given in App.~\ref{app:train_imp}.

\textbf{Sampling.}
\label{sec:inference}
At test time, \methodabbrev draws $\xbm_0$ using \eqref{eq:rto} and returns
$\widehat{\xbm}_1
=
\xbm_0
+
\ubm^{\thetabm}
(\xbm_0,0,1\given\ybm,\sn)$.
Each sample requires one source solve and one network evaluation.
Redrawing $(\epsilonbm,\nbm')$ with $\ybm$ fixed produces conditionally independent samples from the learned model.

\section{Experiments}
\label{sec:exp}
 
\textbf{Datasets.} We evaluate on CelebA ($128{\times}128$), AFHQ-Cat ($256{\times}256$), and multi-coil fastMRI brain AXT2 ($320{\times}320$), averaging all metrics over $100$ test measurements. Details are in App.~\ref{app:datasets}.

\textbf{Forward operators.} We adopt the degradation settings of \cite{martin2025pnpflow} and \cite{pourya2026flower} for natural images, with $\sn=0.05$ everywhere except random inpainting, where $\sn=0.01$. Deblurring uses a $61{\times}61$ Gaussian kernel with $\sigma_b=1.0$ on CelebA and $3.0$ on AFHQ-Cat; super-resolution uses stride subsampling at $2\times$ on CelebA and $4\times$ on AFHQ-Cat; random inpainting masks $70\%$ of pixels; box inpainting uses a centered $40{\times}40$ mask on CelebA and $80{\times}80$ on AFHQ-Cat. For MRI, $\Abm=\Mbm\Fbm\Sbm$ with sampling mask $\Mbm$, Fourier transform $\Fbm$ and per-slice ESPIRiT sensitivities $\Sbm$, at $\times4$ and $\times8$ Cartesian acceleration, with noise scaled to $20$ and $30$\,dB input SNR. Closed forms of \eqref{eq:rto} for each operator are given in App.~\ref{app:sourcesolve}.

\textbf{Baselines.} For natural images, we compare against a supervised MSE regressor trained with the \methodabbrev backbone, the plug-and-play solvers PnP-GS \cite{hurault2022gradient} and PnP-Flow \cite{martin2025pnpflow}, and the iterative posterior samplers DDRM \cite{kawar2022denoising}, DiffPIR \cite{zhu2023diffpir}, OT-ODE \cite{pokle2024training}, Flower \cite{pourya2026flower} and DPS \cite{chung2023diffusion}. Flower and PnP-Flow are run at the  recommended setting, where the exact budget of every method is listed in App.~\ref{app:train_imp}.
For MRI, we compare against zero-filled, wavelet-$\ell_1$ and TV reconstructions, the MSE regressor, and the posterior samplers CSGM \cite{jalal2021robust}, DiffPIR, DPS, DAPS \cite{zhang2025improving} and PnP-DM \cite{wu2024pnpdm, zheng2025inversebench}.
Variational Flow Maps (VFM) \cite{mammadov2026variational} also place the measurement in the source, but learn it with an adapter rather than constructing it in closed form, and report results for latent-space flow models, whereas \methodabbrev operates in pixel space; we compare the two source constructions in App.~\ref{app:vfm}.
NullFlow \cite{shi2026nullflow} is not defined for $\sn>0$ (Sec.~\ref{sec:background}), so we compare against it from $\sn=0$ to $0.05$ in App.~\ref{app:nullflow}.

\textbf{Metrics.} We report PSNR, SSIM \cite{wang2004image}, LPIPS \cite{zhang2018unreasonable} and FID \cite{heusel2017gans}.
None of them says whether the spread of the draws is of the right size, so we also report the calibration ratio $\Delta_{\mathrm{cal}}=V/B$, which divides the variance $V$ of the draws by the squared error $B$ of their mean.
For an exact posterior sampler the two are equal, so $\Delta_{\mathrm{cal}}=1$ is ideal: values below $1$ indicate draws collapsed toward a point estimate, and values above $1$ indicate over-dispersed draws.
It is a second-moment summary and does not certify conditional coverage. We give the definition and estimators in App.~\ref{app:metrics}.

\subsection{Main Results}
\label{sec:res_main}
%%%%%%%%%%%%%%% TABLE CelebA %%%%%%%%%%%%%%%
\begin{table}[t]
\centering
\caption{CelebA, $100$ test images. We report single draw ($M{=}1$) and the average of $M$ draws for \methodabbrev. Note that a single \methodabbrev draw is competitive in LPIPS at one network evaluation, while averaging draws attains the best SSIM on every task and the best PSNR on three of the four. 
\textbf{Best} and \textbf{second best} are shown in \bluehl{blue} and \redhl{red}.}
\label{tab:repro_celeba}
\setlength{\tabcolsep}{3.5pt}
\renewcommand{\arraystretch}{1.12}
\resizebox{0.95\textwidth}{!}{
\begin{tabular}{l c *{12}{c}}
\toprule
\multirow{2}{*}{\textbf{Method}} & \multirow{2}{*}{\textbf{NFE}}
& \multicolumn{3}{c}{\textbf{Deblurring}}
& \multicolumn{3}{c}{\textbf{Super-resolution}}
& \multicolumn{3}{c}{\textbf{Random inpainting}}
& \multicolumn{3}{c}{\textbf{Box inpainting}} \\
\cmidrule(lr){3-5}
\cmidrule(lr){6-8}
\cmidrule(lr){9-11}
\cmidrule(lr){12-14}
& & PSNR$\uparrow$ & SSIM$\uparrow$ & LPIPS$\downarrow$
& PSNR$\uparrow$ & SSIM$\uparrow$ & LPIPS$\downarrow$
& PSNR$\uparrow$ & SSIM$\uparrow$ & LPIPS$\downarrow$
& PSNR$\uparrow$ & SSIM$\uparrow$ & LPIPS$\downarrow$ \\
\midrule \\[-3.5ex] \midrule
Degraded & --
& 27.26 & 0.838 & 0.214
& 11.67 & 0.182 & 0.859
& 11.98 & 0.198 & 1.070 
& 22.33 & 0.754 & 0.217\\
MSE regressor & 1
& \redhl{35.74}& 0.953 & 0.029
& \bluehl{34.17}& \redhl{0.945}& 0.033
& 34.42	& 0.955 & 0.022
& -- & -- & -- \\
PnP-GS & 23     
& 33.97 & 0.924 & 0.041 
& 31.23 & 0.890 & 0.065 
& 29.17 & 0.874 & 0.066
& -- & -- & -- \\
DDRM & 20
& 35.02 & 0.946 & {0.027}
& 32.24	& 0.927 & \redhl{0.027}
& 32.40 & 0.942 & 0.029
& -- & -- & -- \\
DiffPIR & 100
& 34.83 & 0.938 & 0.027
& 31.87 &0.892 & 0.030
&  32.45 & 0.924 & 0.024
&30.48 & 0.918 & 0.028 \\
OT-ODE & 180
& 32.96 & 0.920 & 0.029
& 31.34 & 0.903 & \redhl{0.027}
& 28.69 & 0.871 & 0.049 
& 29.37 & 0.919 & 0.038 \\
PnP-Flow5 & 500
& 34.80 & 0.941 & 0.046
& 31.44 & 0.905 & 0.056
& 33.98 & 0.953 & 0.021 
& 31.06 & 0.939 & 0.043 \\
Flower5-OT & 500
& 35.65 & \redhl{0.954}& 0.031
& 33.03 & 0.931 & 0.039
& 33.97 & 0.952 & 0.019
& 31.85 & 0.952 & 0.023 \\
DPS & 1000
& 34.27 & 0.936 & \redhl{0.019}& 32.23	& 0.917 & \bluehl{0.023}& 33.43 & 0.951 & \bluehl{0.013}& 30.26& 0.942& 0.022 \\
\midrule \\[-3.6ex] \midrule
\textsc{\methodabbrev} ($M{=}1$) & 1
& 32.90  & 0.919	& \bluehl{0.018}&  31.27 & 0.903 & \bluehl{0.023}& 31.97 & 0.927 & 0.018
& 30.73 & 0.934 & 0.021 \\
\textsc{\methodabbrev} ($M{=}4$) & 4
& 34.98 & 0.947	& 0.022
&  33.26 &  0.934& \bluehl{0.023}
& 33.97 & 0.951 & \redhl{0.015}& 32.78 & 0.954 & \bluehl{0.018}\\
\textsc{\methodabbrev} ($M{=}16$) & 16
& 35.68  &  \redhl{0.954} & 	0.028
& 33.95 &  0.943& 0.029
& \redhl{34.64}& \redhl{0.957}& 0.018
& \redhl{33.43}& 	\redhl{0.959}& 	\redhl{0.020}\\
\textsc{\methodabbrev} ($M{=}100$) & 100
& \bluehl{35.91}&  \bluehl{0.956}& 	0.030
& \redhl{34.16}& \bluehl{0.946}& 0.031
& \bluehl{34.84}& \bluehl{0.959}& 0.019
& \bluehl{33.69}& \bluehl{0.961}& 0.022\\
\bottomrule
\end{tabular}}\vspace{-1em}
\end{table}
%%%%%%%%%%%%%%% TABLE CelebA %%%%%%%%%%%%%%%
%%%%%%%%%%%%%%% TABLE FID / calibration %%%%%%%%%%%%%%%
\begin{table}[t]
\centering
\caption{Distributional and calibration metrics for CelebA. $\Delta_\mathrm{cal}$ is the calibration ratio of App.~\ref{app:metrics}, for which $1$ is ideal. Note that \methodabbrev attains the best or second-best FID among the posterior-sampling baselines in one network evaluation.  \textbf{Best} and \textbf{second best} are shown in \bluehl{blue} and \redhl{red}.}
\label{tab:fid_celeba}
\setlength{\tabcolsep}{5pt}
\renewcommand{\arraystretch}{0.9}
\resizebox{0.7\textwidth}{!}{
\begin{tabular}{l c *{6}{c}}
\toprule
\multirow{2}{*}{\textbf{Method}} & \multirow{2}{*}{\textbf{NFE}}
& \multicolumn{2}{c}{\textbf{Gaussian deblurring}}
& \multicolumn{2}{c}{\textbf{Super-resolution}} & \multicolumn{2}{c}{\textbf{Box inpainting}} \\
\cmidrule(lr){3-4}
\cmidrule(lr){5-6}
\cmidrule(lr){7-8}
% & & FID$\downarrow$ & $|\Delta_\mathrm{cal}-3.01|\downarrow$
& & FID$\downarrow$ & $\Delta_\mathrm{cal}\,(\to1)$
& FID$\downarrow$ & $\Delta_\mathrm{cal}\,(\to1)$
& FID$\downarrow$ & $\Delta_\mathrm{cal}\,(\to1)$ \\
\midrule
MSE regressor & 1
& 22.06 & --
& 24.76 & --
& -- & -- \\
DDRM & 20
& 20.26 & 0.31 & \redhl{18.94}& 0.19 & -- & -- \\
Flower1-OT & 100
& \redhl{17.86}& 0.23
& 23.79 & 0.24
& 18.72 & \redhl{0.24} \\
DPS & 1000
& 18.25 & \redhl{0.64} & 19.94 & \redhl{0.67} & \bluehl{15.42}& 2.15 \\
\midrule[\heavyrulewidth]
\textsc{\methodabbrev} (ours) & 1
& \bluehl{15.80}& \bluehl{1.03}& \bluehl{16.80}& \bluehl{0.99}& \redhl{15.53}& \bluehl{1.08} \\
\bottomrule
\end{tabular}}\vspace{-1em}
\end{table}
%%%%%%%%%%%%%%% TABLE FID / calibration %%%%%%%%%%%%%%%

On CelebA (Tab.~\ref{tab:repro_celeba}), a single \methodabbrev draw attains the best LPIPS on Gaussian deblurring, and averaging draws attains the best SSIM on all four operators and the best PSNR on three of the four. Results on AFHQ-Cat are reported in App.~\ref{app:afhq}.
\methodabbrev also attains the best FID on two of the three CelebA operators and the calibration ratio closest to $1$ on all three (Tab.~\ref{tab:fid_celeba}), so the spread of its draws matches the error they make.
On multi-coil fastMRI (Tab.~\ref{tab:mri}) it is competitive with posterior sampler at both accelerations and both noise levels with only \textbf{one} NFE. 
Figures~\ref{fig:qualitative} and~\ref{fig:mri_qual} show the corresponding reconstructions on natural images and on brain MRI.

%%%%%%%%%%%%%%% TABLE fastMRI %%%%%%%%%%%%%%%
\begin{table}[t]
\centering
\caption{fastMRI brain, multi-coil. PSNR and SSIM over $100$ test slices, at $\times4$ and $\times8$ acceleration and two input SNRs.  For \methodabbrev, we report results for a single draw $M=1$ and average of $M$ draws. Note that among posterior samplers \methodabbrev is very competitive even using one NFE. \textbf{Best} and \textbf{second best} are shown in \bluehl{blue} and \redhl{red}.}
\label{tab:mri}
\setlength{\tabcolsep}{3pt}
\renewcommand{\arraystretch}{1.0}
\resizebox{0.8\textwidth}{!}{%
\begin{tabular}{lccccccccc}
\toprule
& & \multicolumn{4}{c}{$\times4$} & \multicolumn{4}{c}{$\times8$} \\
\cmidrule(lr){3-6} \cmidrule(lr){7-10}
& & \multicolumn{2}{c}{$20$\,dB} & \multicolumn{2}{c}{$30$\,dB}
  & \multicolumn{2}{c}{$20$\,dB} & \multicolumn{2}{c}{$30$\,dB} \\
\cmidrule(lr){3-4} \cmidrule(lr){5-6} \cmidrule(lr){7-8} \cmidrule(lr){9-10}
Method & NFE
  & PSNR$\uparrow$ & SSIM$\uparrow$ & PSNR$\uparrow$ & SSIM$\uparrow$
  & PSNR$\uparrow$ & SSIM$\uparrow$ & PSNR$\uparrow$ & SSIM$\uparrow$ \\
\midrule
Zero-filled & ---
&25.35 & 	0.791& 25.37 & 	0.793
& 21.98& 	0.694 & 21.99 & 0.694 \\
Wavelet$+\ell_1$ & ---
& 26.57 & 0.672 & 27.75 & 0.768
& 23.02 & 0.551 & 23.81 & 0.643 \\
TV & ---
& 26.19 & 0.798 & 26.27 & 0.799
& 22.62 & 0.689 & 22.66 & 0.691 \\
MSE regressor & 1
& \bluehl{33.80}& \bluehl{0.909}& \redhl{33.85}& 0.912
& \bluehl{30.71}& \bluehl{0.875}& \bluehl{30.73}& \bluehl{0.878}\\
DiffPIR & 100
& 29.84 & 0.709 & 30.01 & 0.719
& 26.14 & 0.642 & 26.31 & 0.655 \\
CSGM & 1000
& 27.74 & 0.734 & 32.39 & 0.802
& 27.74 & 0.733 & 28.33 & 0.754 \\
DPS & 1000
& 28.81 & 0.646 & 29.83 & 0.669
& 24.89 & 0.575 & 25.17 & 0.574 \\
DAPS & 1000
& 31.55 & 0.818 & 32.44 & 0.782
& 27.31 & 0.749 & 28.01 & 0.714 \\
PnP-DM & 1000
& 32.16 & 0.778 & 33.39 & 0.786
& 27.18 & 0.701 & 29.00 & 0.712 \\
\midrule[\heavyrulewidth]
\textsc{\methodabbrev} ($M{=}1$) & 1
& 32.06 & 0.889 & 32.89 & 0.899
& 28.10 & 0.818 & 29.14 & 0.845 \\
\textsc{\methodabbrev} ($M{=}4$) & 4
& 32.68 & 0.904 & 33.80 & \redhl{0.919}& 29.06 & 0.843 & 29.88 & 0.864 \\
\textsc{\methodabbrev} ($M{=}16$) & 16
& \redhl{32.85}& \redhl{0.908}& \bluehl{34.05}& \bluehl{0.924}& \redhl{29.34}& \redhl{0.850}& \redhl{30.09}& \redhl{0.869}\\
\bottomrule
\end{tabular}%
}\vspace{-1em}
\end{table}
%%%%%%%%%%%%%%% TABLE fastMRI %%%%%%%%%%%%%%%

\begin{figure}[t]
\centering
\includegraphics[width=0.92\textwidth]{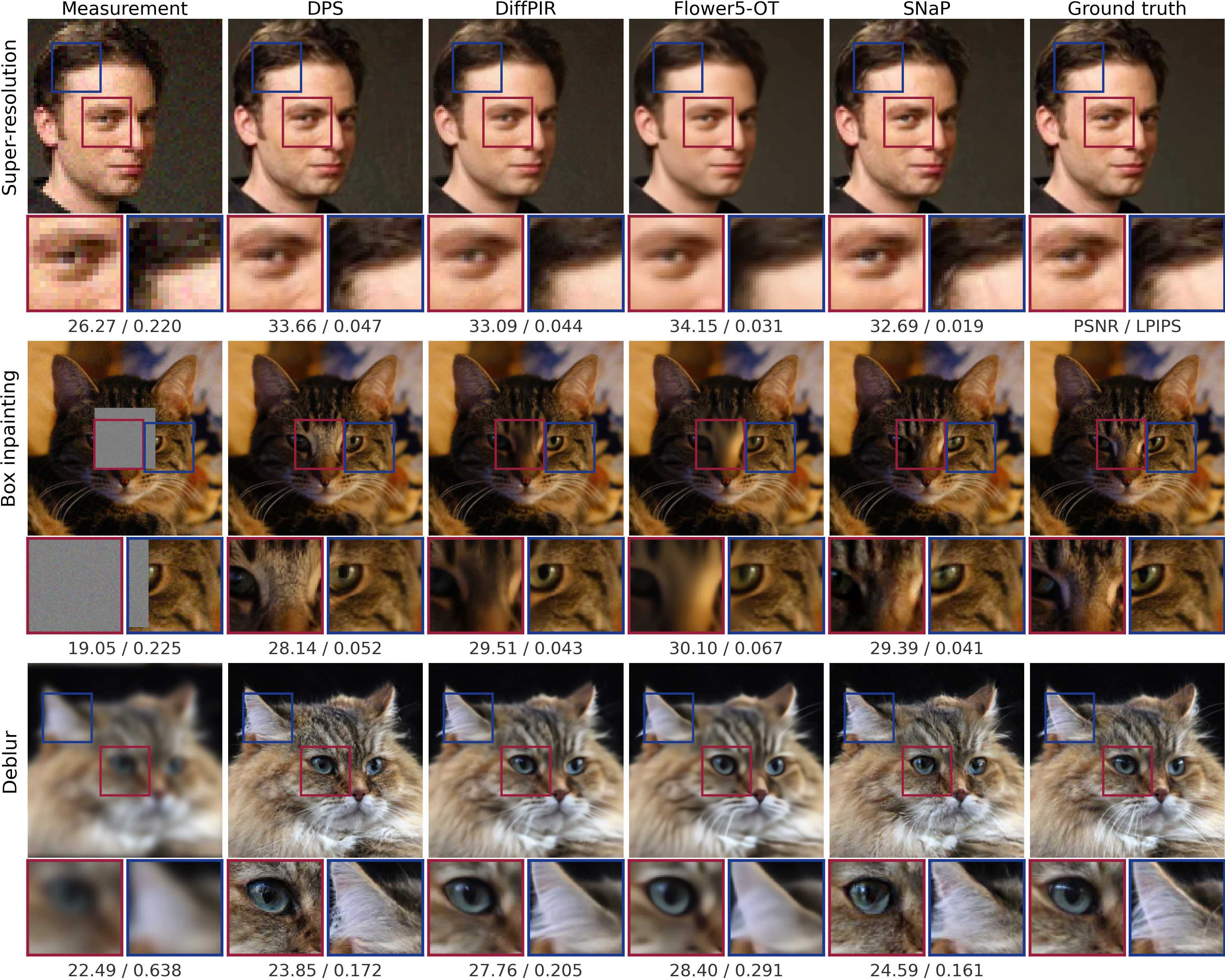}
\vspace{-1em}
\caption{Visual comparison of samples for three imaging tasks, CelebA super-resolution and AFHQ-Cat box inpainting and deblurring, with PSNR and LPIPS for the displayed image. Note that \methodabbrev requires one NFE, while DiffPIR requires $100$ and DPS and Flower require $1000$. \methodabbrev attains the best LPIPS on all three.}
\vspace{-1em}
\label{fig:qualitative}
\end{figure}

\begin{figure}[t]
\centering
\includegraphics[width=\textwidth]{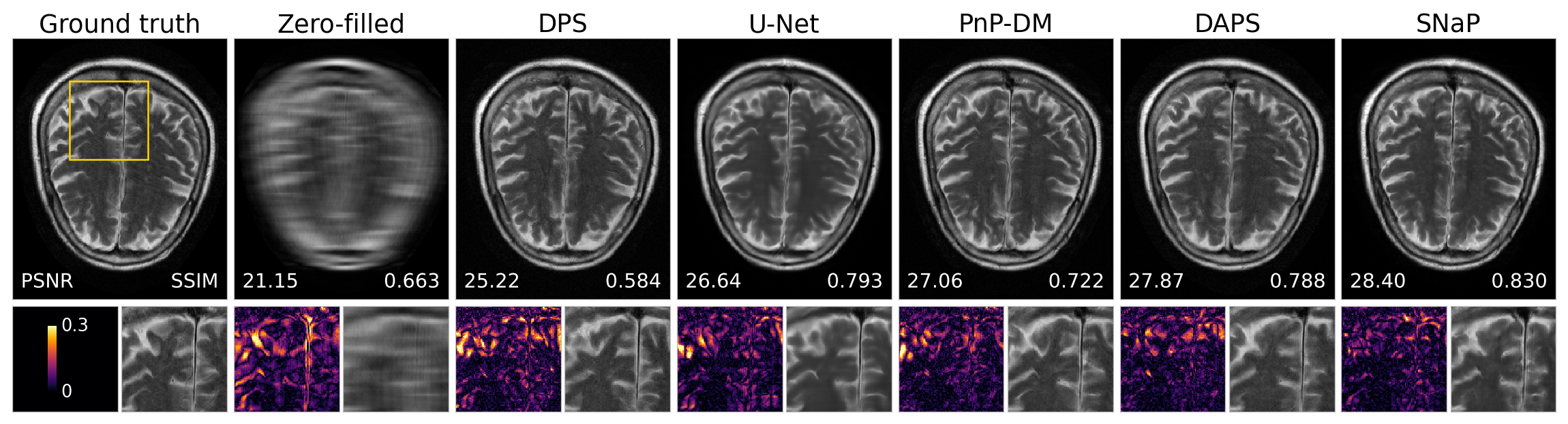}
\vspace{-1em}
\caption{fastMRI brain, multi-coil, $\times 8$ acceleration at $20$\,dB input SNR. Reconstructions with PSNR and SSIM, and error magnitude and a zoomed in region is displayed. Note that \methodabbrev achieves competitive reconstruction using 1 NFE, while baselines use $1000$.}
\vspace{-1.5em}
\label{fig:mri_qual}
\end{figure}

\subsection{Computational cost}
\label{sec:costmain}

\begin{table}[t]
\centering\footnotesize
\setlength{\abovecaptionskip}{4pt}
\setlength{\belowcaptionskip}{2pt}
\caption{\textbf{Computational cost.} NFE and 1-batch wall-clock time per posterior sample. \textbf{Best} in each row is in bold.}
\label{tab:cost_comput}
\setlength{\tabcolsep}{4pt}         % horizontal padding between columns
\renewcommand{\arraystretch}{0.8}   % row height
\resizebox{0.6\textwidth}{!}{%      % overall width; the font scales with it
\begin{tabular}{lccccc}
\toprule
\multicolumn{6}{c}{\emph{CelebA, Gaussian deblurring}} \\
\midrule
 & DPS & OT-ODE & Flower & DDRM & \textsc{\methodabbrev} \\
\midrule
NFE $\downarrow$              & $1000$ & $180$ & $100$ & $20$ & $\mathbf{1}$ \\
Time / sample (s) $\downarrow$ & 27.014 & 6.575 & 2.614 & 0.360 & \textbf{0.012} \\
Slowdown vs.\ ours $\downarrow$ & $2251\times$ & $548\times$ & $218\times$ & $30\times$ & -- \\
\midrule
\multicolumn{6}{c}{\emph{fastMRI brain, multi-coil $\times 8$}} \\
\midrule
 & CSGM & DAPS & PnP-DM & DPS & \textsc{\methodabbrev} \\
\midrule
NFE $\downarrow$              & $1000$ & $1000$ & $1000$ & $1000$ & $\mathbf{1}$ \\
Time / sample (s) $\downarrow$ & 127.78 & 96.28 & 92.02 & 76.07 & \textbf{0.227} \\
Slowdown vs.\ ours $\downarrow$ & $563\times$ & $424\times$ & $405\times$ & $335\times$ & -- \\
\bottomrule
\end{tabular}%
}
\vspace{-1.5em}
\end{table}

Each \methodabbrev sample costs one network evaluation and one source draw by \eqref{eq:rto}, whose solve with $(\Abm\Abm^\top+\lambda\Ibm) ^{-1}$ is diagonal (or Fourier diagonal) for the image operators and reduces to a short conjugate-gradient solve for multi-coil MRI (App.~\ref{app:sourcesolve}).
As reported in Tab.~\ref{tab:cost_comput}, this makes \methodabbrev $30$ to $2250\times$ faster per sample than the iterative samplers on CelebA deblurring and $335$ to $563\times$ faster on $\times8$ MRI.

\subsection{Ablations}
\label{sec:ablations}

\textbf{Source.} 
Table~\ref{tab:ablsource} isolates the effect of the source by comparing our proposed source to two isotropic sources: the first is the usual flow matching source $\Ncal(\zerobm, \Ibm)$ and the other is $\Ncal(\zerobm, \tau^2\Ibm)$ to control for the influence of $\tau$.
We hold the other design choices constant and $\ybm$ and $\sn$ are retained in the conditioning.
Under the isotropic sources, the calibration ratio falls to $0.00$: the draws are nearly identical, yielding a \emph{point estimator} with higher single-draw PSNR and worse LPIPS.
The measurement-adapted source keeps the draws dispersed, and its single-draw advantage is perceptual rather than distortion-based.
Figures~\ref{fig:source_draws} and~\ref{fig:source_vs_m} in App.~\ref{app:ablsource} show this collapse directly---the draws and their per-pixel spread---and trace both metrics against the averaging budget $M$.

\begin{table}[t]
\centering\footnotesize
\setlength{\abovecaptionskip}{4pt}
\setlength{\belowcaptionskip}{2pt}
\caption{\textbf{Source distribution.} CelebA, Gaussian deblurring at $\sn=0.05$; only the source differs across rows. $M$ is the number of averaged draws, and $\Delta_{\mathrm{cal}}$ is the calibration ratio (App.~\ref{app:metrics}), for which $1$ is ideal. An isotropic source collapses onto the posterior mean, where draws are near-identical and averaging brings no gain.}
\label{tab:ablsource}
\setlength{\tabcolsep}{4pt}         % horizontal padding between columns
\renewcommand{\arraystretch}{0.9}   % row height
\resizebox{0.7\textwidth}{!}{%      % overall width; the font scales with it
\begin{tabular}{lccccc}
\toprule
 & \multicolumn{2}{c}{$M=1$} & \multicolumn{2}{c}{$M=100$} & \\
\cmidrule(lr){2-3}\cmidrule(lr){4-5}
Source & PSNR$\uparrow$ & LPIPS$\downarrow$ & PSNR$\uparrow$ & LPIPS$\downarrow$ & $\Delta_\mathrm{cal}\,(\to1)$ \\
\midrule
$\Ncal(\zerobm,\Ibm)$              & 35.85 & 0.030 & 35.86 & 0.030 & 0.00 \\
$\Ncal(\zerobm,\tau^{2}\Ibm)$      & 35.83 & 0.030 & 35.84 & 0.030 & 0.00 \\
$\Ncal(\myv,\tau^{2}\Wbm)$ (ours)  & 32.90 & \textbf{0.018} & \textbf{35.91} & 0.030 & \textbf{1.03} \\
\bottomrule
\end{tabular}%
}
\vspace{-1.5em}
\end{table}

\textbf{Additional results.} App.~\ref{app:afhq} reports the full AFHQ-Cat tables and App.~\ref{app:visual} shows further reconstructions on natural images and brain MRI.
App.~\ref{app:nullflow} compares \methodabbrev with NullFlow across a range of measurement noise levels, and App.~\ref{app:vfm} compares the two source constructions on the two-dimensional benchmark of \cite{mammadov2026variational}.
App.~\ref{app:toy} verifies the method against a closed-form posterior, and App.~\ref{app:ablations} reports the remaining ablations, on the working-prior scale $\tau$, the $(r,t)$ sampling ratio and few-step composition.

\section{Related work}
\label{sec:related}

\textbf{Iterative posterior samplers.}
Most generative solvers enforce consistency with the measurement repeatedly along the sampling trajectory, either by correcting the generative dynamics with a data-fidelity gradient or a pseudoinverse update \cite{chung2023diffusion,song2023pseudoinverseguided,kawar2022denoising,wang2023zeroshot,pokle2024training,benhamu2024dflow,kim2025flowdps}, or by alternating a generative step with an explicit data-consistency step \cite{zhu2023diffpir,chung2024decomposed,zhang2025improving,martin2025pnpflow,pourya2026flower}.
These methods reach high sample quality, but a single posterior sample costs tens to thousands of network evaluations.

\textbf{One-step generative models.}
Consistency models \cite{song2023consistency,boffi2026build}, shortcut models \cite{frans2025shortcut} and MeanFlow \cite{geng2025meanflow} eliminate the trajectory and map a source sample to the data in a single network evaluation.
Using them for inverse problems requires the measurement to enter through the network conditioning or through the source, since no intermediate state is available at which to enforce data consistency.
For real-world super-resolution, MFSR \cite{wang2026mfsr} distills a text-to-image flow model into a MeanFlow in which the low-resolution image enters only through conditioning, without an explicit forward model or posterior sampling.

\textbf{Measurement-dependent sources.}
Bridge models transport from a measurement-dependent source by starting from the degraded observation itself \cite{liu2023i2sb,zhou2024ddbm,delbracio2023inversion}, but that source carries no per-direction covariance.
NullFlow \cite{shi2026nullflow} restricts a MeanFlow to $\nullsp(\Abm)$ and  VFM \cite{mammadov2026variational} learns a measurement-dependent adapter jointly with the flow map, which covers a broader class of problems but yields a source that is trained. MeanFlow has also been applied to PDE-based Bayesian inverse problems with a prior-aligned base measure that does not depend on the measurement \cite{li2026functional}.

\section{Conclusion}
\label{sec:conclusion}
We introduced \methodabbrev, a one-step posterior sampler for linear inverse problems with Gaussian noise.
Instead of correcting toward the measurement at every step, \methodabbrev builds it into the transport.
The transport starts from a distribution centered at the linear estimate that minimizes the expected transport distance, with a spread that matches the remaining uncertainty.
We proved that the exact one-step map carries this distribution to the target posterior.
On natural images and multi-coil MRI, \methodabbrev draws a posterior sample in a single network evaluation, $30$ to $2250\times$ faster than iterative samplers.
On CelebA, averaging its draws gives the best SSIM on all four operators and the best PSNR on three.
Its draws also reach the best FID on two of the three operators tested and the calibration ratio closest to $1$ on all three.

\paragraph{Limitations.}
The construction  is tied to linear operators, since the working posterior is available in closed form only in that setting, and nonlinear forward models require further work.
Sampling the source requires one solve with $\Abm\Abm^\top+\lambda\Ibm$, which is inexpensive for the operators considered here but may be demanding for operators without exploitable structure.
Finally, \methodabbrev operates in pixel space, and extending the source construction to latent-space flow models remains open.

\section*{Acknowledgments}
This work was supported in part by the National Science Foundation under CAREER award \mbox{CCF-2625643} and award \mbox{CCF-2622128}.

\bibliography{utils/refs}
\bibliographystyle{utils/IEEEbib}
\appendix
\newpage

\section{Proofs}\label{app:proofs}

\subsection{Proof of Proposition~\ref{prop:optimal-source}}
\label{app:proof-optimal-source}

\begin{proof}
Let $\xbm_0=\bm K\ybm+\bm\xi$ as in \eqref{eq:source-family}.

Since $\xbm_1-\xbm_0=(\xbm_1-\bm K\ybm)-\bm\xi$ and $\bm\xi$ has mean $\zerobm$ and is independent of $\xbm_1-\bm K\ybm$, for every unit vector $\bm e$,
\begin{equation*}
\E\big[(\bm e^\top(\xbm_1-\xbm_0))^2\big]=\bm e^\top\bm R(\bm K)\,\bm e+\bm e^\top\bm\Sigma\,\bm e,
\qquad
\E[\|\xbm_1-\xbm_0\|_2^2]=\tr\bm R(\bm K)+\tr\bm\Sigma .
\end{equation*}
For fixed $\bm\Sigma$, only $\bm R(\bm K)$ depends on $\bm K$.

Let $\bm C_{\xbm\ybm}\defn\E[\xbm_1\ybm^\top]=\bm C\Abm^\top$ and $\bm C_{\ybm}\defn\E[\ybm\ybm^\top]=\Abm\bm C\Abm^\top+\sn^2\Ibm$, which is positive definite since $\sn>0$, so that $\bm K_{\bm C}=\bm C_{\xbm\ybm}\bm C_{\ybm}^{-1}$.
Expanding $\bm R(\bm K)=\bm C-\bm K\bm C_{\xbm\ybm}^\top-\bm C_{\xbm\ybm}\bm K^\top+\bm K\bm C_{\ybm}\bm K^\top$ and using $\bm K_{\bm C}\bm C_{\ybm}=\bm C_{\xbm\ybm}$,
\begin{equation}
\bm R(\bm K)=\bm R(\bm K_{\bm C})+(\bm K-\bm K_{\bm C})\,\bm C_{\ybm}\,(\bm K-\bm K_{\bm C})^\top,
\qquad
\bm R(\bm K_{\bm C})=\bm C-\bm C_{\xbm\ybm}\bm C_{\ybm}^{-1}\bm C_{\xbm\ybm}^\top=\bm C-\bm K_{\bm C}\Abm\bm C .
\label{eq:lmmse-square}
\end{equation}

By \eqref{eq:lmmse-square}, for every unit vector $\bm e$,
\begin{equation*}
\bm e^\top\bm R(\bm K)\,\bm e=\bm e^\top\bm R(\bm K_{\bm C})\,\bm e+\big\|\bm C_{\ybm}^{1/2}(\bm K-\bm K_{\bm C})^\top\bm e\big\|_2^2\;\ge\;\bm e^\top\bm R(\bm K_{\bm C})\,\bm e,
\end{equation*}
so $\bm K_{\bm C}$ minimizes the displacement along every direction.
Taking the trace,
\begin{equation*}
\tr\bm R(\bm K)=\tr\bm R(\bm K_{\bm C})+\big\|\bm C_{\ybm}^{1/2}(\bm K-\bm K_{\bm C})^\top\big\|_F^2 ,
\end{equation*}
and since $\bm C_{\ybm}^{1/2}$ is invertible, the second term vanishes if and only if $\bm K=\bm K_{\bm C}$.
The total displacement is therefore minimized if and only if $\bm K=\bm K_{\bm C}$, for every $\bm\Sigma$.

\emph{Isotropic case.}
For $\bm C=\tau^2\Ibm$ and $\lambda=\sn^2/\tau^2$, dividing by $\tau^2$ gives $\bm K_{\bm C}=\Abm^\top(\Abm\Abm^\top+\lambda\Ibm)^{-1}$, so $\bm K_{\bm C}\ybm=\myv$, and
$\bm R(\bm K_{\bm C})=\tau^2\big(\Ibm-\Abm^\top(\Abm\Abm^\top+\lambda\Ibm)^{-1}\Abm\big)=\tau^2\Wbm$.
\end{proof}

\paragraph{Remarks.}
(i) The proof uses only the first two moments of $\xbm_1$, $\nbm$ and $\bm\xi$; neither the images, the noise, nor the perturbation need to be Gaussian.
(ii) If $\xbm_1\sim\Ncal(\zerobm,\bm C)$, Gaussian conditioning gives $\xbm_1\given\ybm\sim\Ncal(\bm K_{\bm C}\ybm,\bm\Sigma_{\bm C})$, so the source with center $\bm K_{\bm C}\ybm$ and Gaussian perturbation of covariance $\bm\Sigma_{\bm C}$ is the exact posterior; for $\bm C=\tau^2\Ibm$ it is the working posterior \eqref{eq:working-posterior}.
(iii) For $\bm C=\tau^2\Ibm$, as $\sn\to0$ with $\tau$ fixed, $\lambda\to0$, and the standard limit of Tikhonov regularization gives $\bm K_{\bm C}\to\pinv{\Abm}$ and $\Wbm\to\Ibm-\pinv{\Abm}\Abm=\Pbm$, so the source tends to $\Ncal(\pinv{\Abm}\ybm,\tau^2\Pbm)$.

\subsection{Proof of Proposition~\ref{prop:invariance}}
\label{app:proof-invariance}

\begin{proof}
Let $p_t(\cdot\given\ybm)$ denote the distribution of $\zbm_t=(1-t)\xbm_0+t\xbm_1$, so that $p_0(\cdot\given\ybm)=q_\tau(\cdot\given\ybm)$ and $p_1(\cdot\given\ybm)=p(\cdot\given\ybm)$.
The marginal velocity
\begin{equation}
\vbm_t(\zbm,\ybm)=\E\!\left[\xbm_1-\xbm_0\given\zbm_t=\zbm,\ybm\right]
\label{eq:marginal-velocity}
\end{equation}
generates $\{p_t(\cdot\given\ybm)\}_{t\in[0,1]}$ through the continuity equation \cite{lipman2023flow,albergo2025stochastic}.

Since $\sn>0$, we have $\lambda>0$ and $w_i\in(0,1]$ for all $i$, so $\Wbm$ is positive definite.
For $t<1$, the interpolant $\zbm_t$ is therefore the sum of $t\xbm_1$ and an independent Gaussian with nondegenerate covariance $(1-t)^2\tau^2\Wbm$.
Its density $p_t(\cdot\given\ybm)$ is thus smooth and strictly positive.
Since $\xbm_1$ has finite second moments, standard arguments show that $\vbm_t(\cdot,\ybm)$ is locally Lipschitz, uniformly for $t\in[0,T]$ and every $T<1$ \cite{albergo2025stochastic}.
The ODE $\tfrac{\dd}{\dd t}\zbm_t=\vbm_t(\zbm_t,\ybm)$ then has a unique solution on $[0,1)$, and its flow map satisfies $(\psi_{0\to t})_{\#}\,q_\tau(\cdot\given\ybm)=p_t(\cdot\given\ybm)$ for every $t<1$.

Since $\E[\|\zbm_t-\xbm_1\|_2^2]=(1-t)^2\,\E[\|\xbm_1-\xbm_0\|_2^2]\to0$, the distribution $p_t(\cdot\given\ybm)$ converges to $p(\cdot\given\ybm)$ as $t\to1$.
By the terminal-limit assumption of Proposition~\ref{prop:invariance}, $\psi_{0\to1}(\xbm)\defn\lim_{t\to1}\psi_{0\to t}(\xbm)$ exists for $q_\tau$-almost every $\xbm$.
Almost-sure convergence implies convergence in distribution, so
\begin{equation}
(\psi_{0\to1})_{\#}\,q_\tau(\cdot\given\ybm)=p(\cdot\given\ybm).
\label{eq:flow-endpoint}
\end{equation}

Integrating the ODE from $0$ to $t$ and letting $t\to1$ gives
\begin{equation}
\psi_{0\to1}(\xbm_0)-\xbm_0=\int_0^1\vbm_s\big(\psi_{0\to s}(\xbm_0),\ybm\big)\,\dd s=\ubm(\xbm_0,0,1\given\ybm),
\label{eq:avgvel-anchor}
\end{equation}
where the last equality is the average velocity in \eqref{eq:avgvel}.
Hence $\psi_{0\to1}=T_{\ybm}$, and \eqref{eq:flow-endpoint} proves \eqref{eq:invariance}.
The target $p(\cdot\given\ybm)$ does not depend on $\tau$, so the conclusion holds for every $\tau>0$.
\end{proof}

\subsection{Perturb-and-solve sampling}
\label{app:rto}

Direct sampling would require a square root of $\Wbm$. The perturb-and-solve
sampler avoids this square root while preserving the same Gaussian law.

\begin{lemma}
\label{lem:rto}
Let $\epsilonbm\sim\Ncal(\zerobm,\tau^{2}\Ibm)$ and $\nbm'\sim\Ncal(\zerobm,\sn^{2}\Ibm)$ be independent of each other and of $(\xbm_1,\nbm)$, and let $\xbm_0$ be given by \eqref{eq:rto}.
Then $\xbm_0\given\ybm\sim\Ncal(\myv,\tau^{2}\Wbm)$ for every $\Abm\in\R^{m\times n}$ and every $\sn>0$.
\end{lemma}

\begin{proof}
Write $\Kbm\defn\Abm^\top(\Abm\Abm^\top+\lambda\Ibm)^{-1}$, so that $\myv=\Kbm\ybm$ and $\Wbm=\Ibm-\Kbm\Abm$ by Proposition~\ref{prop:optimal-source}.
Substituting these into \eqref{eq:rto} gives
\begin{equation}
    \xbm_0=\myv+\Wbm\epsilonbm-\Kbm\nbm'.
\label{eq:rto-affine}
\end{equation}
Given $\ybm$, this is an affine function of two independent Gaussian vectors, so $\xbm_0$ is Gaussian with mean $\myv$ and covariance $\tau^2\Wbm^2+\sn^2\Kbm\Kbm^\top$.
Both $\Wbm$ and $\Kbm\Kbm^\top$ are diagonal in the right singular basis of $\Abm$, with entries $w_i$ and $s_i^2/(s_i^2+\lambda)^2$.
Using $\sn^2=\tau^2\lambda$, the $i$-th diagonal entry of the covariance is
\[
    \tau^2w_i^2+\tau^2\lambda\frac{s_i^2}{(s_i^2+\lambda)^2}
    =\tau^2\frac{\lambda^2+\lambda s_i^2}{(s_i^2+\lambda)^2}
    =\tau^2\frac{\lambda}{s_i^2+\lambda}
    =\tau^2w_i ,
\]
so the covariance equals $\tau^2\Wbm$.
\end{proof}

Each source draw requires applications of $\Abm$ and $\Abm^\top$ and one solve
with $\Abm\Abm^\top+\lambda\Ibm$. No square root of $\Wbm$ is required.
Operator-specific implementations are given in App.~\ref{app:sourcesolve}.

\section{Additional Results}
\label{app:ablations}

\subsection{Full results on AFHQ-Cat}
\label{app:afhq}
%%%%%%%%%%%%%%% TABLE AFHQ-Cat %%%%%%%%%%%%%%%
\begin{table}[!ht]
\centering
\caption{AFHQ-Cat, $100$ test images.  We report  single draw ($M{=}1$) and the average of $M$ draws for \methodabbrev. Note that a single \methodabbrev draw is competitive in LPIPS at one network evaluation, and averaging a few draws recovers distortion at a cost still far below the iterative samplers. On AFHQ-Cat the number of time steps used by PnP-Flow and Flower depends on the operator, so their NFE does too; the per-operator budgets are listed in App.~\ref{app:train_imp}. The MSE regressor is omitted on box inpainting, where its single output is the blurred conditional mean of the masked region and is not comparable with the other methods. \textbf{Best} and \textbf{second best} are shown in \bluehl{blue} and \redhl{red}.}
\label{tab:repro_afhq}
\setlength{\tabcolsep}{3.5pt}
\renewcommand{\arraystretch}{1.12}
\resizebox{\textwidth}{!}{
\begin{tabular}{l c *{12}{c}}
\toprule
\multirow{2}{*}{\textbf{Method}} & \multirow{2}{*}{\textbf{NFE}}
& \multicolumn{3}{c}{\textbf{Deblurring}}
& \multicolumn{3}{c}{\textbf{Super-resolution}}
& \multicolumn{3}{c}{\textbf{Random inpainting}}
& \multicolumn{3}{c}{\textbf{Box inpainting}} \\
\cmidrule(lr){3-5}
\cmidrule(lr){6-8}
\cmidrule(lr){9-11}
\cmidrule(lr){12-14}
& & PSNR$\uparrow$ & SSIM$\uparrow$ & LPIPS$\downarrow$
& PSNR$\uparrow$ & SSIM$\uparrow$ & LPIPS$\downarrow$
& PSNR$\uparrow$ & SSIM$\uparrow$ & LPIPS$\downarrow$
& PSNR$\uparrow$ & SSIM$\uparrow$ & LPIPS$\downarrow$ \\
\midrule \\[-3.5ex] \midrule
Degraded & --
& 24.39 & 0.532 &  0.543 
& 11.98 & 0.212 & 0.900
& 13.25 & 0.214 & 1.091
& 21.57 & 0.736 & 0.214 \\
MSE regressor & 1
&\bluehl{29.82}& 	\bluehl{0.806}& 0.289
& \bluehl{29.16}& 	\bluehl{0.817}& 0.201 
& 33.23	 & 0.915 & 0.066
& -- & -- & -- \\
PnP-GS & 23     
& 28.39 & 0.787 & 0.387 
& 24.44 & 0.639 & 0.411 
& 29.89 & 0.841 & 0.123
& -- & -- & --\\
PnP-Flow5 & 500--2500
& 29.01& 0.785 & 0.312 
& 28.01 & 0.791 & 0.167 
& \bluehl{34.47}& \bluehl{0.933}& \redhl{0.044}& 27.14 & 0.900 & 0.127 \\
DDRM & 20
& 29.01 & 0.784 & 0.194
& 27.09 & 0.781 & 0.186
& 32.20 & 0.907 & 0.064
& -- & -- & --\\
DiffPIR & 100
& 28.89 & 0.773 & 0.187
& 23.16 & 0.635 & 0.291
& 31.76 & 0.881 & 0.053
& 27.71 & 0.880 & 0.059\\
OT-ODE & 180
&  27.82 & 0.735 & \bluehl{0.126}& 26.71 & 0.737 & \bluehl{0.108}& 29.98 & 0.849 & 0.086
& 24.47 & 0.873 & 0.095 \\
Flower5-OT & 500--2500
& \redhl{29.73} & \redhl{0.801} &0.264 
& 27.09 & 0.767& 0.272 
&  \redhl{34.41}& \redhl{0.932}& 0.047
& 27.32 & \redhl{0.922} & 0.067\\
DPS & 1000
&25.90 & 0.661 &  \redhl{0.148}& 25.98 & 0.694 & 0.144
& 32.06 & 0.896 &  \bluehl{0.036}& 26.78 & 0.900 &  \redhl{0.054}\\
\midrule \\[-3.6ex] \midrule
\textsc{\methodabbrev} ($M{=}1$) & 1
& 26.25 & 0.674 & 0.159
&26.07 & 0.703 & 0.130
&30.48 & 0.853 & 0.066
&26.46 & 0.891 & \redhl{0.054}
\\
\textsc{\methodabbrev} ($M{=}4$) & 4
& 28.33 & 0.750 & 0.157
& 28.07 & 	0.777 & \redhl{0.120}& 32.38 &	0.896 &	0.052
& 28.49 & 0.914 & \bluehl{0.047}\\
\textsc{\methodabbrev} ($M{=}16$) & 16
& 29.09 & 0.779 & 0.231
& 28.77 &  	0.802 & 	0.157
& 33.04 & 0.909 & 	0.056
& \redhl{29.32}& 	\redhl{0.922}& 	0.056\\
\textsc{\methodabbrev} ($M{=}100$) & 100
& 29.32 & 0.788 & 0.320
& \redhl{28.99}& \redhl{0.810}& 0.187
& 33.25 & 0.913 & 0.058
& \bluehl{29.53}& \bluehl{0.925}& 0.069\\
\bottomrule
\end{tabular}}
\end{table}
%%%%%%%%%%%%%%% TABLE AFHQ-Cat %%%%%%%%%%%%%%%
Table~\ref{tab:repro_afhq} reports the same protocol as Tab.~\ref{tab:repro_celeba} on AFHQ-Cat.

\subsection{Additional visual results}
\label{app:visual}
Figure~\ref{fig:panel_celeba_box} shows box inpainting on CelebA.

\begin{figure}[ht]
\centering
\includegraphics[width=0.8\textwidth]{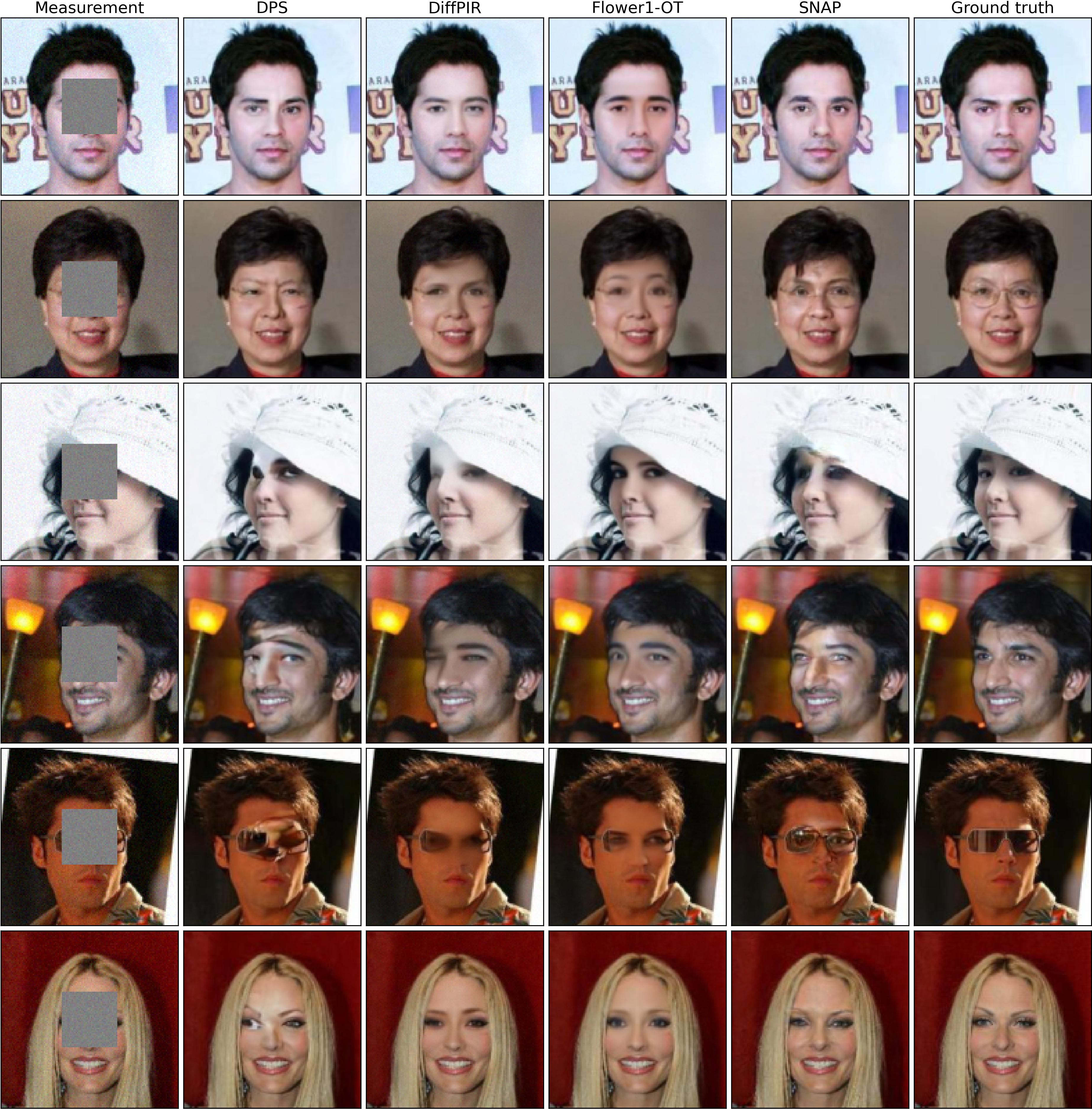}
\caption{Box inpainting on CelebA.}
\label{fig:panel_celeba_box}
\end{figure}

Figure~\ref{fig:panel_afhq_box} shows box inpainting on AFHQ-Cat.

\begin{figure}[ht]
\centering
\includegraphics[width=0.8\textwidth]{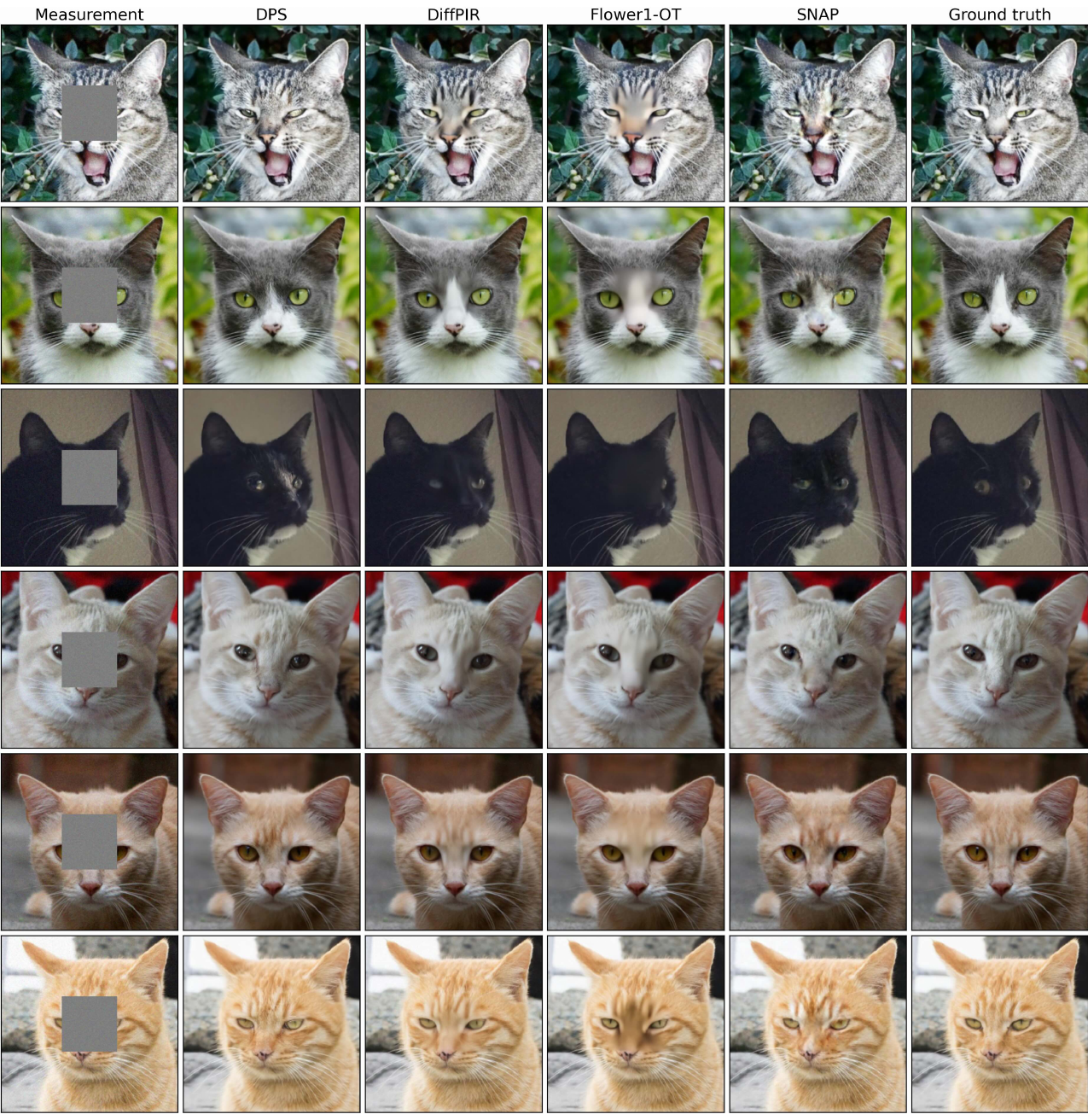}
\caption{Box inpainting on AFHQ-Cat.}
\label{fig:panel_afhq_box}
\end{figure}

Figure~\ref{fig:panel_mri} shows multi-coil fastMRI at $\times8$ acceleration and $30$\,dB input SNR.

\begin{figure}[ht]
\centering
\includegraphics[width=\textwidth]{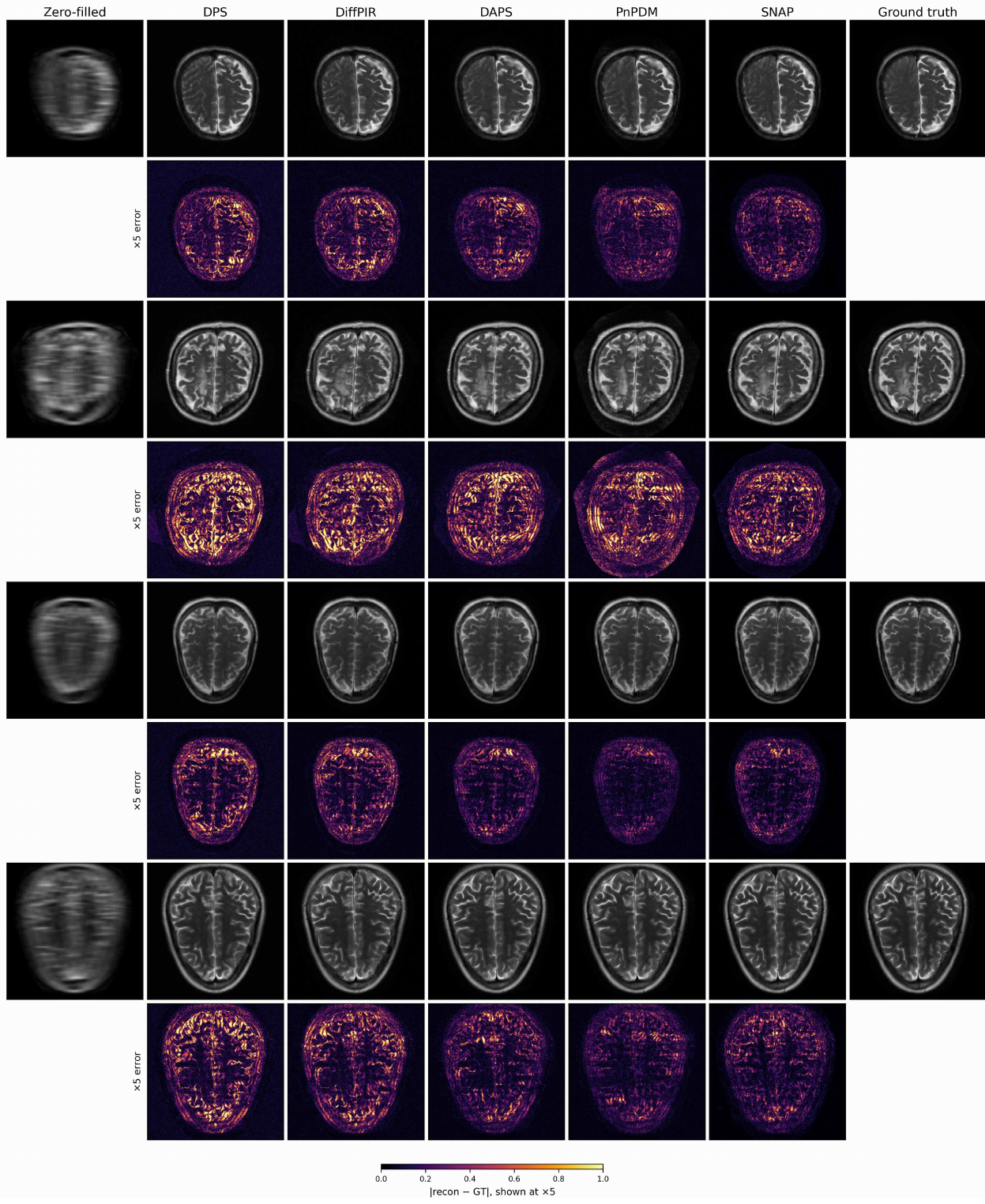}
\caption{fastMRI brain, multi-coil, $\times8$ acceleration at $30$\,dB input SNR.}
\label{fig:panel_mri}
\end{figure}

Figure~\ref{fig:panel_celeba_blur} shows Gaussian deblurring on CelebA.

\begin{figure}[ht]
\centering
\includegraphics[width=0.8\textwidth]{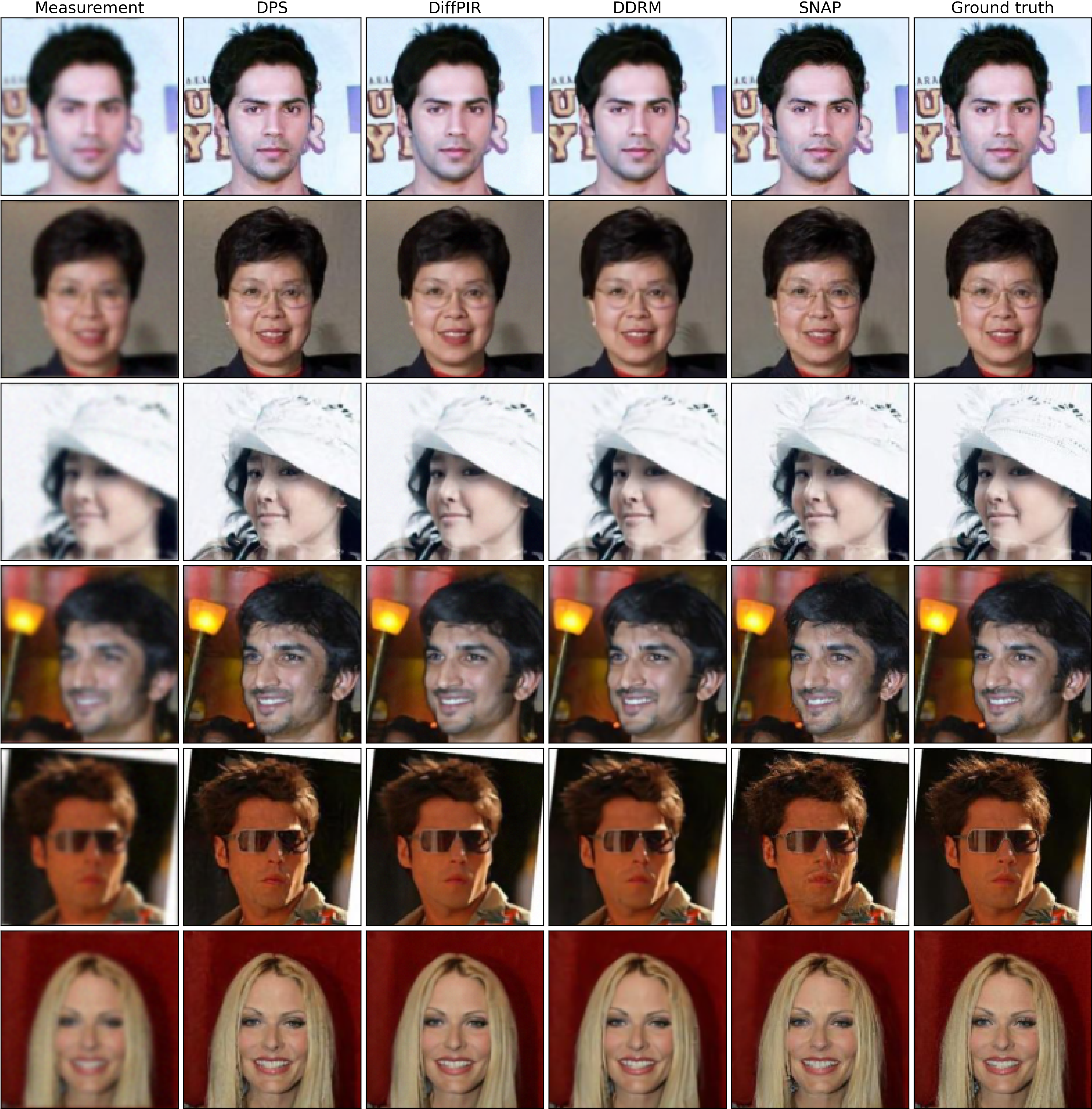}
\caption{Gaussian deblurring on CelebA.}
\label{fig:panel_celeba_blur}
\end{figure}

Figure~\ref{fig:panel_draws} shows several \methodabbrev draws for box inpainting.

\begin{figure}[ht]
\centering
\includegraphics[width=\textwidth]{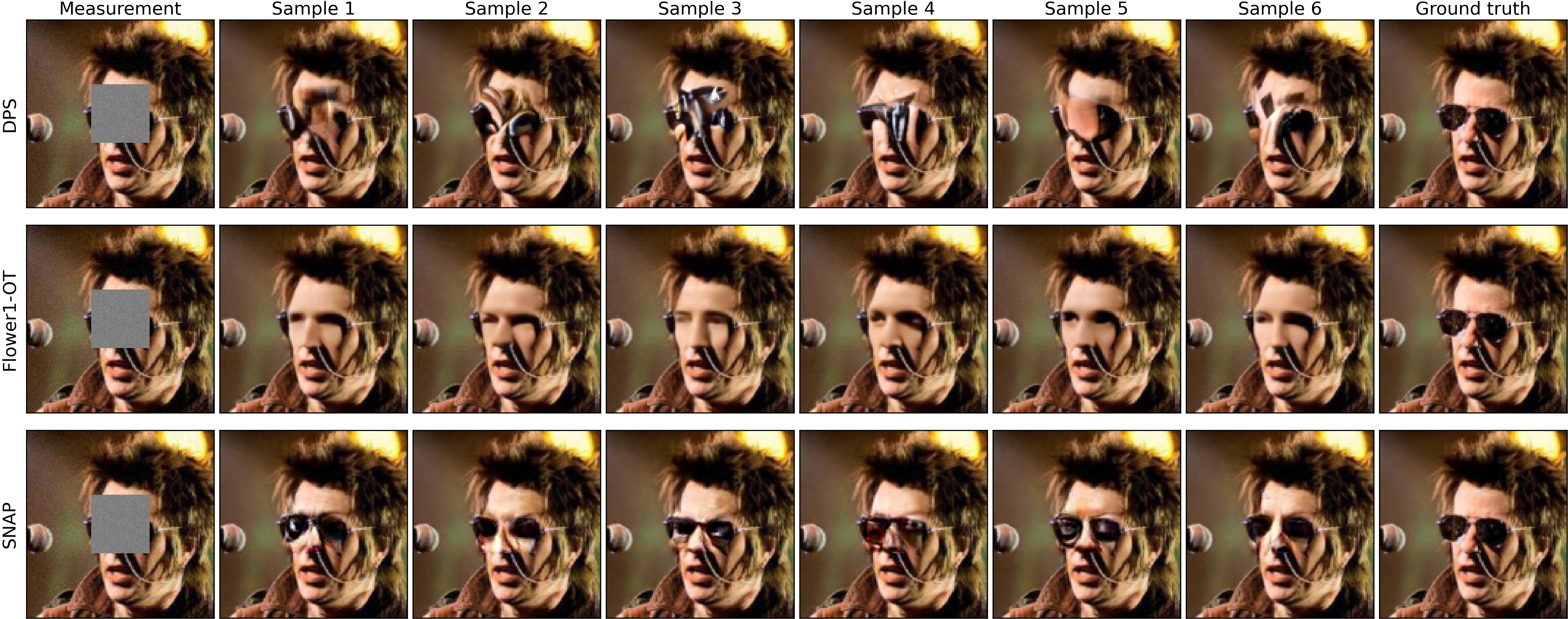}
\caption{Independent \methodabbrev draws for box inpainting.}
\label{fig:panel_draws}
\end{figure}

Figure~\ref{fig:panel_draws3} shows further \methodabbrev draws for box inpainting.

\begin{figure}[ht]
\centering
\includegraphics[width=\textwidth]{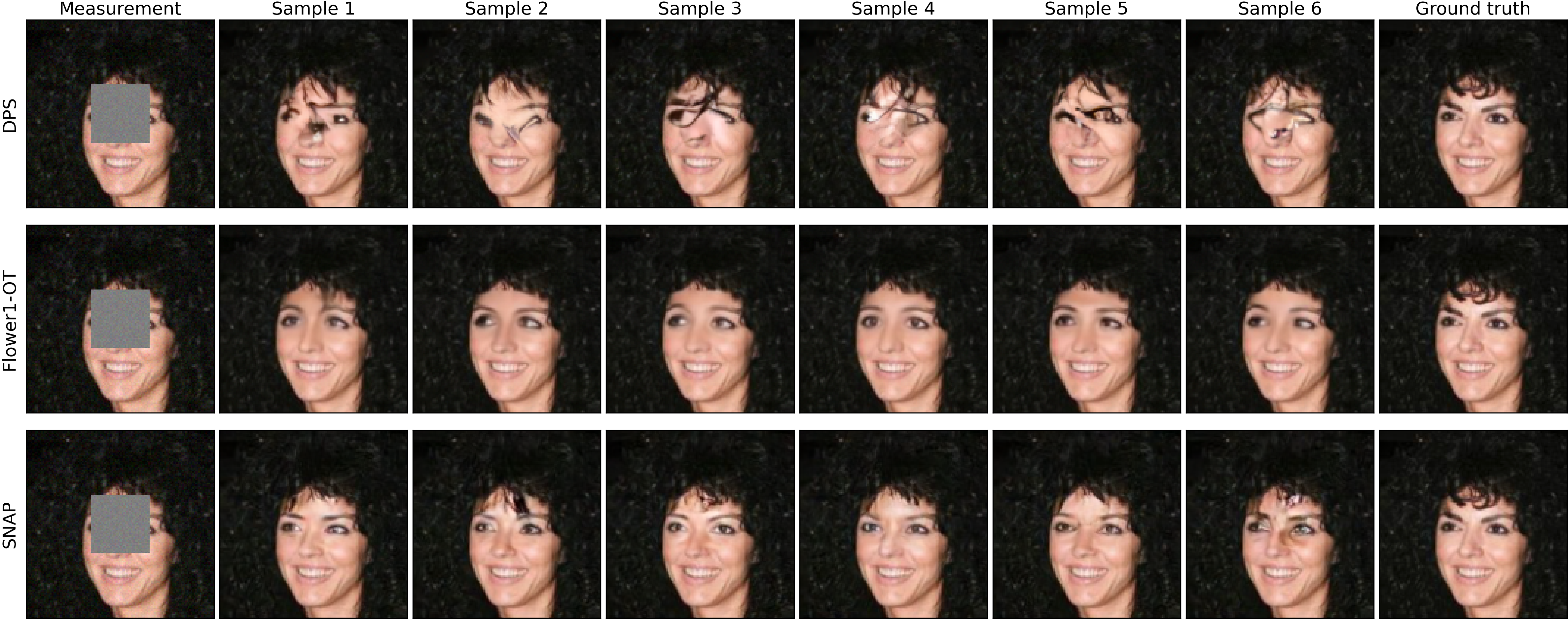}
\caption{Further independent \methodabbrev draws for box inpainting.}
\label{fig:panel_draws3}
\end{figure}

Figure~\ref{fig:panel_draws2} shows further \methodabbrev draws for box inpainting.

\begin{figure}[ht]
\centering
\includegraphics[width=\textwidth]{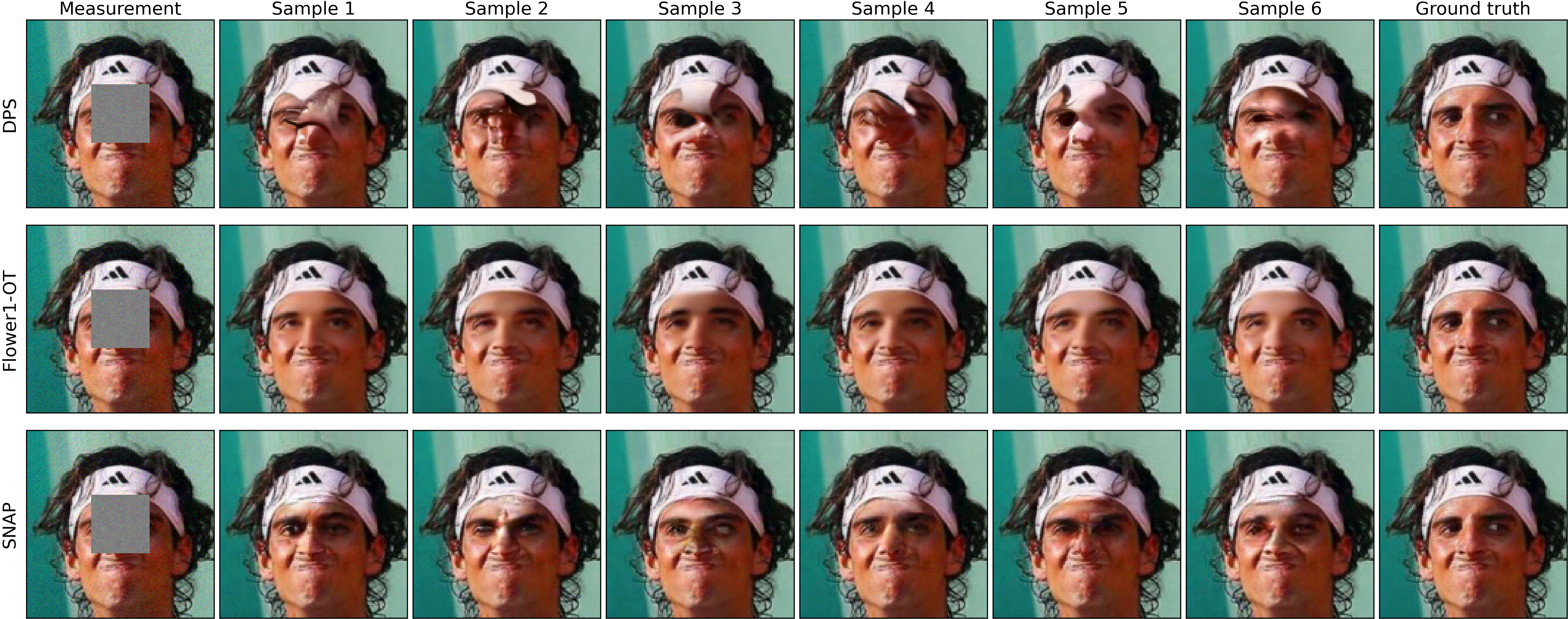}
\caption{Further independent \methodabbrev draws for box inpainting.}
\label{fig:panel_draws2}
\end{figure}

\subsection{Ablations}
\paragraph{Noise misspecification.}\label{app:noise}
Since $\sn$ enters the source in closed form, a wrong value at test time changes how the source allocates uncertainty rather than breaking the method: by \eqref{eq:source-spectrum}, over-estimating $\sn$ loosens the source toward the working prior, while under-estimating it shrinks the source toward the least-squares anchor.
Table~\ref{tab:ablmisspec} tests the first case: the model and its source keep $\sn=0.05$, while the test measurements carry noise between $0.01$ and $0.05$, an over-estimate of up to a factor of $5$.
Reconstruction degrades gracefully. PSNR still rises as the measurements get cleaner, from $32.90$ to $34.29$\,dB, and LPIPS stays between $0.017$ and $0.020$, although the gain over the input shrinks from $5.09$ to $4.33$\,dB as the mismatch grows.
Over this range the method therefore does not require $\sn$ to be known exactly; under-estimation is not tested here.

\begin{table}[t]
\centering\small
\caption{CelebA Gaussian deblurring  trained at $\sn=0.05$ and evaluated on $y$ with different measurement noise level. ${\dagger}$ marks the matched case, where the test noise equals the training noise.}
\label{tab:ablmisspec}
\begin{tabular}{lcccccc}
\toprule
Measurement noise $\sn$& PSNR$\uparrow$ & SSIM$\uparrow$ & LPIPS$\downarrow$ & input PSNR  & Gain \\
\midrule
$0.01$            &34.29 & 0.940 & 0.020 & 29.96 &	+4.33\\
$0.02$            &34.15 & 0.939 & 0.019 & 29.62 &	+4.53\\
$0.03$            &33.89 & 0.936 & 0.018 & 29.12 &	+4.77\\
$0.04$            &33.50 & 0.930 & 0.017 & 28.52 &  +4.98\\
$0.05^{\dagger}$  &32.90 & 0.919 & 0.018 & 27.81 &	+5.09\\
\bottomrule
\end{tabular}
\end{table}

\paragraph{Working-prior scale $\tau$.}
Proposition~\ref{prop:invariance} states that $\tau$ shapes the source and the transport but not the target, so it is a design parameter rather than a modelling choice.
Table~\ref{tab:abltau} reports CelebA Gaussian deblurring for $\tau\in\{0.15,0.5,1.0\}$.
Distortion is indistinguishable between $\tau=0.15$ and $\tau=0.5$ at every $M$ and falls by about $0.4$\,dB at $\tau=1.0$, while LPIPS is best at $\tau=0.5$.
The relative spread of the draws is $5.49\%$, $5.48\%$ and $5.72\%$ for the three settings, essentially unchanged, as the invariance predicts.
The method therefore does not require $\tau$ to be tuned precisely, and we use the per-operator values listed in App.~\ref{app:train_imp}.

\begin{table}[t]
\centering\small
\caption{\textbf{Working-prior scale $\tau$.} CelebA Gaussian deblurring, $100$ test measurements, input PSNR $27.24$\,dB. $M$ is the number of averaged draws}
\label{tab:abltau}
\setlength{\tabcolsep}{5pt}
\renewcommand{\arraystretch}{0.95}
\begin{tabular}{lccccccccc}
\toprule
& \multicolumn{3}{c}{$\tau=0.15$} & \multicolumn{3}{c}{$\tau=0.5$} & \multicolumn{3}{c}{$\tau=1.0$} \\
\cmidrule(lr){2-4}\cmidrule(lr){5-7}\cmidrule(lr){8-10}
$M$ & PSNR$\uparrow$ & SSIM$\uparrow$ & LPIPS$\downarrow$ & PSNR$\uparrow$ & SSIM$\uparrow$ & LPIPS$\downarrow$ & PSNR$\uparrow$ & SSIM$\uparrow$ & LPIPS$\downarrow$ \\
\midrule
$1$   & \textbf{32.941} & \textbf{0.9198} & 0.0190 & 32.922 & 0.9187 & \textbf{0.0186} & 32.512 & 0.9060 & 0.0244 \\
$4$   & \textbf{34.966} & \textbf{0.9463} & 0.0209 & 34.948 & 0.9457 & \textbf{0.0181} & 34.588 & 0.9385 & 0.0197 \\
$16$  & \textbf{35.688} & \textbf{0.9537} & 0.0275 & 35.663 & 0.9532 & \textbf{0.0234} & 35.323 & 0.9475 & 0.0249 \\
$100$ & \textbf{35.913} & \textbf{0.9558} & 0.0300 & 35.885 & 0.9553 & \textbf{0.0256} & 35.552 & 0.9500 & 0.0273 \\
\bottomrule
\end{tabular}
\end{table}

\paragraph{Sampling ratio $(r,t)$.}
The fraction $p_{\mathrm{ratio}}\defn\Pr[r{=}t]$ of training samples drawn with $r=t$ controls how often the objective reduces to plain conditional flow matching.
Table~\ref{tab:ablratio} reports CelebA Gaussian deblurring at $\tau=0.15$ for $p_{\mathrm{ratio}}\in\{0.25,0.5,0.75\}$.
Setting $p_{\mathrm{ratio}}=0.25$ is worse at every $M$, by about $0.3$\,dB PSNR at a single draw, while $0.5$ and $0.75$ are indistinguishable within the precision reported.
We therefore use $p_{\mathrm{ratio}}=0.5$ throughout.
The trend in $M$ is the same for all three settings, with PSNR and SSIM improving as draws are averaged and LPIPS degrading, as the average moves toward the posterior mean.

\begin{table}[t]
\centering\small
\caption{\textbf{$(r,t)$ sampling ratio.} CelebA Gaussian deblurring, $\tau=0.15$, $100$ test measurements, input PSNR $27.24$\,dB. $M$ is the number of averaged draws.}
\label{tab:ablratio}
\setlength{\tabcolsep}{5pt}
\renewcommand{\arraystretch}{0.95}
\begin{tabular}{lccccccccc}
\toprule
& \multicolumn{3}{c}{$p_{\mathrm{ratio}}=0.25$} & \multicolumn{3}{c}{$p_{\mathrm{ratio}}=0.5$} & \multicolumn{3}{c}{$p_{\mathrm{ratio}}=0.75$} \\
\cmidrule(lr){2-4}\cmidrule(lr){5-7}\cmidrule(lr){8-10}
$M$ & PSNR$\uparrow$ & SSIM$\uparrow$ & LPIPS$\downarrow$ & PSNR$\uparrow$ & SSIM$\uparrow$ & LPIPS$\downarrow$ & PSNR$\uparrow$ & SSIM$\uparrow$ & LPIPS$\downarrow$ \\
\midrule
$1$   & 32.567 & 0.9130 & 0.0216 & 32.941 & 0.9198 & 0.0190 & 32.976 & 0.9210 & 0.0189 \\
$4$   & 34.679 & 0.9429 & 0.0197 & 34.966 & 0.9463 & 0.0209 & 34.976 & 0.9465 & 0.0206 \\
$16$  & 35.438 & 0.9513 & 0.0266 & 35.688 & 0.9537 & 0.0275 & 35.682 & 0.9536 & 0.0268 \\
$100$ & 35.676 & 0.9537 & 0.0296 & 35.913 & 0.9558 & 0.0300 & 35.902 & 0.9557 & 0.0292\\
\bottomrule
\end{tabular}
\end{table}

\paragraph{Few-step composition.}
Composing $\ubm_{r,t}$ over a uniform partition of $[0,1]$ gives $k$-step sampling with no retraining, and Table~\ref{tab:ablsteps} reports $k\in\{1,2,4,8\}$ on CelebA Gaussian deblurring.
More steps are not monotonically better.
All metrics peak at $k=2$ and decline afterwards, and $k=8$ is worse than a single step for $M\ge16$.
The model is trained to jump from the source to the data, so discretizing the trajectory evaluates $\ubm^{\thetabm}$ at intermediate $(r,t)$ pairs at which its error compounds across steps rather than cancelling.
The gain from $k=2$ is small, so we report $k=1$ throughout.

\begin{table}[t]
\centering\footnotesize
\setlength{\abovecaptionskip}{4pt}
\setlength{\belowcaptionskip}{2pt}
\caption{\textbf{Trajectory steps $k$.} CelebA Gaussian deblurring, $\tau=0.15$, $100$ test measurements, input PSNR $27.24$\,dB. $M$ is the number of averaged draws. \textbf{Best} in each row and metric is in bold.}
\label{tab:ablsteps}
\setlength{\tabcolsep}{3pt}
\renewcommand{\arraystretch}{0.9}
\resizebox{\textwidth}{!}{%
\begin{tabular}{lcccccccccccc}
\toprule
& \multicolumn{3}{c}{$k=1$} & \multicolumn{3}{c}{$k=2$} & \multicolumn{3}{c}{$k=4$} & \multicolumn{3}{c}{$k=8$} \\
\cmidrule(lr){2-4}\cmidrule(lr){5-7}\cmidrule(lr){8-10}\cmidrule(lr){11-13}
$M$ & PSNR$\uparrow$ & SSIM$\uparrow$ & LPIPS$\downarrow$ & PSNR$\uparrow$ & SSIM$\uparrow$ & LPIPS$\downarrow$ & PSNR$\uparrow$ & SSIM$\uparrow$ & LPIPS$\downarrow$ & PSNR$\uparrow$ & SSIM$\uparrow$ & LPIPS$\downarrow$ \\
\midrule
$1$   & 32.941 & 0.9198 & 0.0190 & \textbf{33.084} & \textbf{0.9211} & \textbf{0.0172} & 33.037 & 0.9205 & \textbf{0.0172} & 32.992 & 0.9200 & 0.0173 \\
$4$   & 34.966 & 0.9463 & 0.0209 & \textbf{35.044} & \textbf{0.9469} & \textbf{0.0206} & 35.001 & 0.9465 & \textbf{0.0206} & 34.942 & 0.9458 & 0.0210 \\
$16$  & 35.688 & 0.9537 & \textbf{0.0275} & \textbf{35.729} & \textbf{0.9540} & \textbf{0.0275} & 35.688 & 0.9536 & 0.0276 & 35.621 & 0.9529 & 0.0281 \\
$100$ & 35.913 & 0.9558 & \textbf{0.0300} & \textbf{35.943} & \textbf{0.9561} & 0.0302 & 35.902 & 0.9557 & 0.0303 & 35.833 & 0.9550 & 0.0308 \\
\bottomrule
\end{tabular}}
\end{table}

\paragraph{Source distribution.}\label{app:ablsource}
Table~\ref{tab:ablsource} isolates the source quantitatively; Figures~\ref{fig:source_draws} and~\ref{fig:source_vs_m} show the same comparison sample by sample and as a function of the averaging budget $M$.
All three settings share the network, the conditioning, and the training schedule, differing only in the distribution the one-step map starts from: our measurement-adapted posterior source at $\tau=0.15$ versus isotropic Gaussian sources at $\tau=0.15$ and $\tau=1.0$.

\begin{figure}[ht]
\centering
\includegraphics[width=\textwidth]{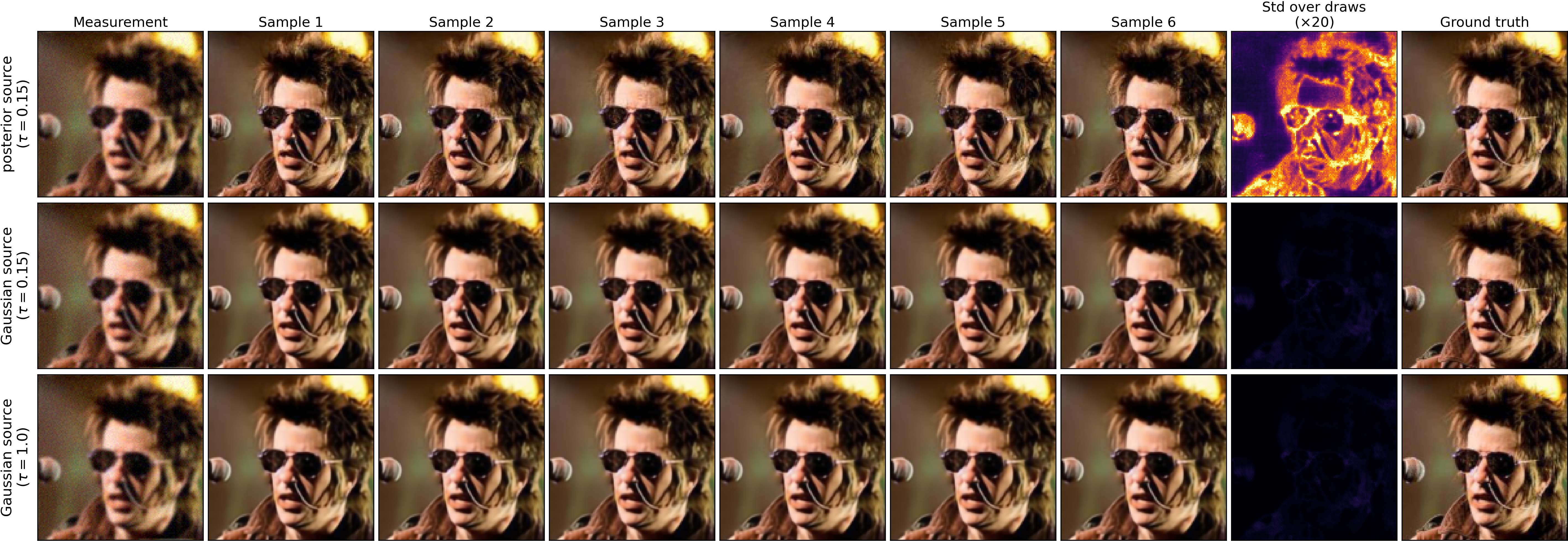}
\caption{\textbf{Source distribution, qualitative.} Six independent draws from a fixed CelebA Gaussian-deblurring measurement under the three sources, with the per-pixel standard deviation over draws on the same $\times 20$ scale in every row. The posterior source varies where the measurement is uninformative; both Gaussian sources collapse onto a single reconstruction.}
\label{fig:source_draws}
\end{figure}

Figure~\ref{fig:source_draws} shows six draws per source for one measurement.
Under the proposed posterior source the draws differ visibly in the hair, the skin texture and the region occluded by the sunglasses, and the standard deviation over draws concentrates on exactly the edges and high-frequency structures that the blur has suppressed, leaving the smooth background, which the measurement already determines, near zero.
Under either isotropic Gaussian source the six draws are visually indistinguishable and the standard-deviation map is essentially black at the same $\times 20$ gain: the map has become a point estimator that happens to take a random input.
Raising the Gaussian scale from $\tau=0.15$ to $\tau=1.0$ does not restore the spread, so the shape of the source is what dictates the spread rather than its size.
Figure~\ref{fig:source_vs_m} makes the consequence quantitative.
\begin{figure}[ht]
\centering
\includegraphics[width=0.75\textwidth]{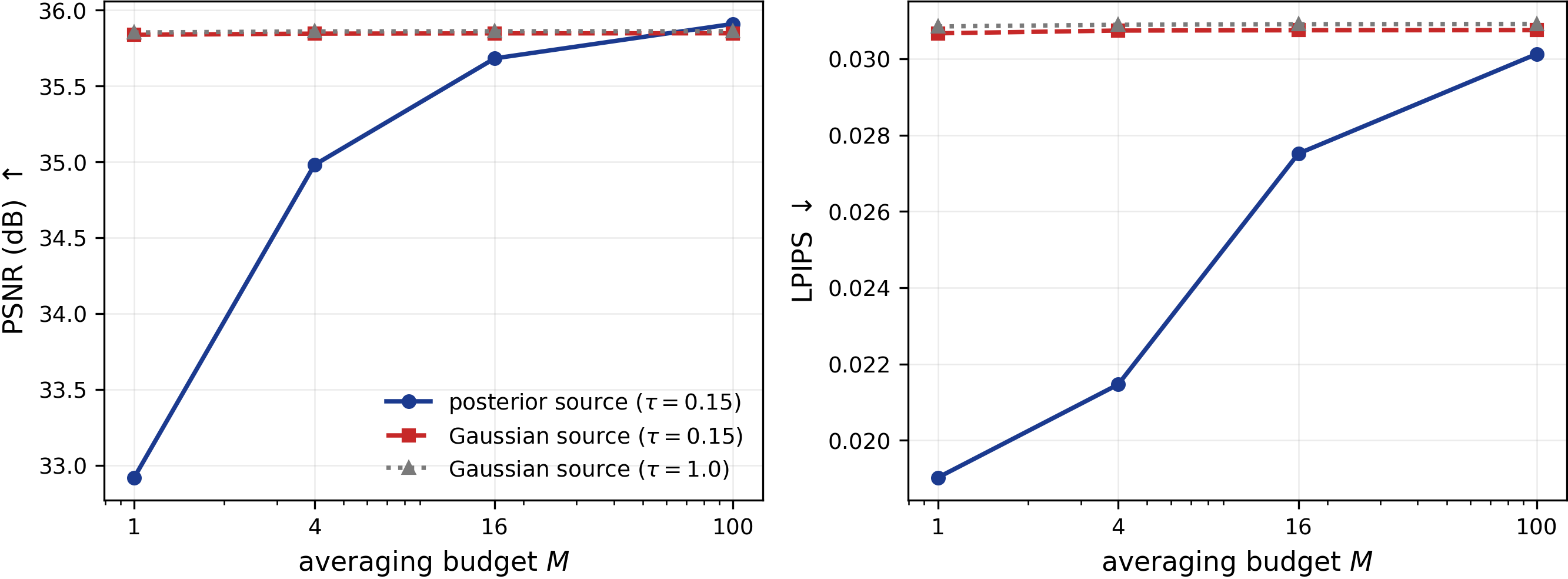}
\caption{\textbf{Source distribution, averaging budget.} PSNR (left, higher is better) and LPIPS (right, lower is better) against the number of averaged draws $M$ on CelebA Gaussian deblurring. Only the posterior source responds to $M$; the Gaussian sources are flat because their draws carry almost no variance to average away.}
\label{fig:source_vs_m}
\end{figure}

\subsection{Comparison with VFM}
\label{app:vfm}
VFM \cite{mammadov2026variational} also places the measurement in the source, but learns it: an adapter $q_\phi(\zbm\given\ybm)$ is trained jointly with an unconditional flow map.
Its results are reported in the latent space of a pretrained autoencoder, whereas \methodabbrev operates in pixel space, so the published numbers are not directly comparable.
We therefore compare the two constructions on the 2D benchmark of \cite{mammadov2026variational}. The prior is a $4{\times}4$ checkerboard on $[-2,2]^2$ with $20$k training samples, observed through $\ybm=\bm{a}^\top\xbm+\varepsilon$ with $\varepsilon\sim\Ncal(0,\sigma^2)$ and $\sigma=0.1$.
We use two forward operators that differ only in the direction of $\bm{a}$, an aligned operator $\bm{a}=(1,0)$, which is the one used in the original paper, and a rotated operator $\bm{a}=(\cos\tfrac{\pi}{5},\sin\tfrac{\pi}{5})$.
All methods share one backbone, a $6$-layer width-$512$ SiLU MLP in $\xbm_1$-prediction form with random-Fourier embeddings of $r$, $t$ and $\sigma$, and one budget, AdamW with learning rate $2{\times}10^{-4}$, batch $2048$, $50$k steps and EMA $0.999$.
The only thing that varies is where the measurement enters the generative process.  The learned-adapter baseline follows VFM.
The flow map never sees $\ybm$, and all measurement information is carried by a diagonal Gaussian adapter $q_\phi(\zbm\given\ybm)=\Ncal(\bm\mu_\phi(\ybm),\diag\bm\sigma_\phi^2(\ybm))$ trained jointly with the flow map under the variational objective, with an observation term, a KL term and an EMA copy of the weights for stability.
We also report a frozen-$\thetabm$ ablation in which the pretrained flow map is held fixed and only $q_\phi$ is trained.
\methodabbrev instead keeps the flow map conditioned on $\ybm$ and draws its source in closed form by \eqref{eq:rto}, with no adapter and no auxiliary losses.

The source covariance $\Wbm=\Ibm-\bm{a}\bm{a}^\top/(\norm{\bm{a}}_2^2+\lambda)$ is diagonal in the basis $\{\bm{a},\bm{a}^\perp\}$, and in the coordinate basis only when the two coincide.
For $\bm{a}=(1,0)$ they coincide, so a diagonal adapter is exactly expressive and the aligned operator is a parity check.
Rotating $\bm{a}$ breaks the coincidence while holding the prior, the noise level, the architecture and the training budget fixed.
The resulting gap can be quantified before any training: the KL divergence from $\Ncal(\myv,\tau^2\Wbm)$ to its best diagonal approximation is $0$ nats for the aligned operator at every $\tau$, and $1.576$ nats for the rotated  at $\tau=1$.

\begin{figure}[t]
\centering
\includegraphics[width=0.85\textwidth]{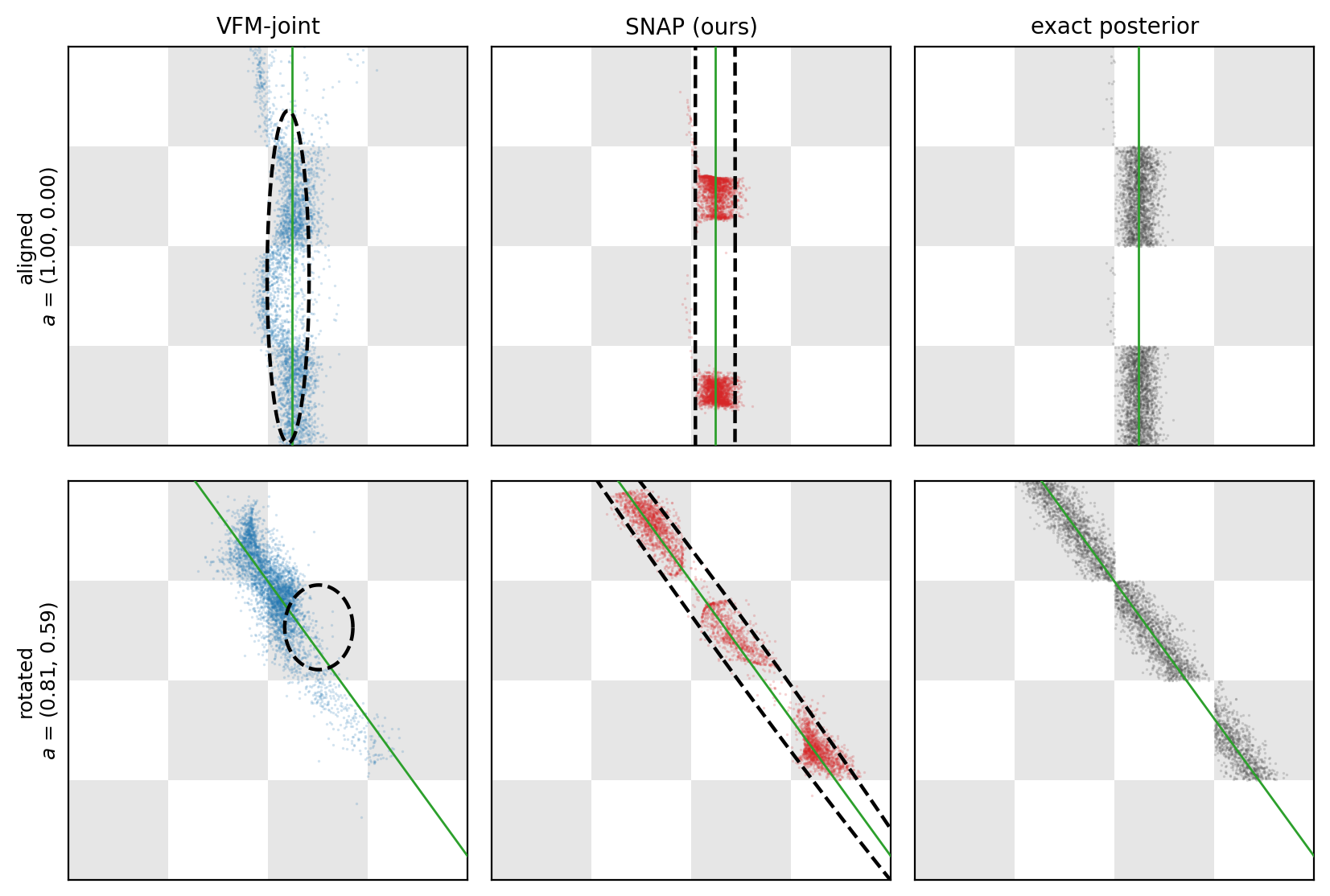}
\caption{One-step ($K{=}1$) posterior samples on the $4{\times}4$ checkerboard. Rows are the two forward operators and columns are the learned adapter of VFM, the \methodabbrev source, and the exact rejection-sampled posterior. The green line is the measurement constraint $\bm{a}^\top\xbm=\ybm$, and the dashed ellipse on each method panel is that method's $2\sigma$ source covariance, the learned diagonal $q_\phi(\zbm\given\ybm)$ for VFM and the closed-form $\tau^2\Wbm$ for \methodabbrev. In the top row the measurement direction is axis-aligned, both sources align with the constraint, and both recover the posterior. In the bottom row the constraint tilts. The \methodabbrev source tilts with it, whereas the diagonal adapter cannot and degenerates to a near-isotropic blob, so its samples concentrate on one mode and spill off the checkerboard support.}
\label{fig:vfm2d}
\end{figure}

\begin{table}[t]
\centering\small
\caption{Posterior MMD$^2$ ($\downarrow$) on the two-dimensional checkerboard benchmark at $K{=}1$, mean $\pm$ standard deviation over three seeds, each at its validation-selected setting. The aligned operator is a parity check, where a diagonal adapter is exactly expressive, and the rotated operator differs only in the direction of $\bm{a}$.}
\label{tab:vfm2d}
\begin{tabular}{lcc}
\toprule
Method & Aligned (parity) & Rotated (discriminating) \\
\midrule
\methodabbrev & $0.0027\pm0.0003$ & $0.0171\pm0.0012$ \\
VFM (joint) & $0.0109\pm0.0007$ & $0.0737\pm0.0199$ \\
VFM (frozen $\thetabm$) & $0.0102\pm0.0003$ & $0.1900\pm0.0015$ \\
\bottomrule
\end{tabular}
\end{table}

On the aligned operator, where the diagonal adapter can represent the exact source, all three methods are close, with \methodabbrev already the most accurate (Table~\ref{tab:vfm2d}).
On the rotated operator the two constructions separate.
The jointly trained adapter degrades by roughly a factor of seven and the frozen variant by more than an order of magnitude, while \methodabbrev changes far less, since its source rotates with the operator by construction.
The gaps are significant under a Welch test against the jointly trained adapter, with $p=0.0005$ on the aligned operator and $p=0.038$ on the rotated one.
Figure~\ref{fig:vfm2d} shows the mechanism, and the effect is also visible in the support of the samples, with a support accuracy of $0.97\pm0.004$ for \methodabbrev against $0.81\pm0.03$ for the jointly trained adapter on the rotated operator.

\subsection{Comparison with NullFlow}
\label{app:nullflow}
NullFlow is defined only for noiseless measurements, so it cannot enter the
benchmarks of Sec.~\ref{sec:exp}. We instead sweep the test noise on CelebA box
inpainting from $\sigma_n=0$, where NullFlow is exact, to the $\sigma_n=0.05$ of
our main experiments (Table~\ref{tab:nullflow}). NullFlow is trained noiselessly
with an LPIPS and feature loss, and SNaP at $\sigma_n=0.05$ with the standard
MeanFlow loss, so LPIPS favors NullFlow.

\begin{table}[t]
\centering\small
\caption{NullFlow vs.\ SNaP on CelebA box inpainting ($M{=}1$) as test noise grows.}
\label{tab:nullflow}
\begin{tabular}{l ccc c ccc}
\toprule
& \multicolumn{3}{c}{NullFlow} & & \multicolumn{3}{c}{\methodabbrev} \\
\cmidrule{2-4}\cmidrule{6-8}
Test $\sn$ & PSNR$\uparrow$ & SSIM$\uparrow$ & LPIPS$\downarrow$ & & PSNR$\uparrow$ & SSIM$\uparrow$ & LPIPS$\downarrow$ \\
\midrule
$0.00$ & $33.50$ & $0.977$ & $0.012$ & & $31.37$ & $0.947$ & $0.025$ \\
$0.01$ & $33.27$ & $0.970$ & $0.012$ & & $31.35$ & $0.947$ & $0.024$ \\
$0.02$ & $32.65$ & $0.952$ & $0.016$ & & $31.30$ & $0.946$ & $0.023$ \\
$0.05$ & $30.11$ & $0.855$ & $0.055$ & & $30.74$ & $0.934$ & $0.021$ \\
\bottomrule
\end{tabular}
\end{table}

Table~\ref{tab:nullflow} reports both methods on CelebA box inpainting, using a single draw and $\lambda$ frozen at each model's training value.
At $\sn=0$, NullFlow is the stronger method, which is expected, since its source is exact in the noiseless limit and every state satisfies $\Abm\xbm=\ybm$.
Its accuracy falls quickly as noise is added, and by $\sn=0.05$ its SSIM has dropped from $0.977$ to $0.855$ and its LPIPS has more than quadrupled.
This is the failure mode identified in Sec.~\ref{sec:background}, where the least-squares anchor carries $\pinv{\Abm}\nbm$ that a null-space flow cannot remove.
\methodabbrev changes little over the same range and is the better method at $\sn=0.05$ on all three metrics.
The two methods therefore occupy different regimes, with NullFlow suited to noiseless measurements and \methodabbrev to noisy ones.

\subsection{Irreducible regression error: \methodabbrev{} vs.\ isotropic sources}
\label{app:irreducible}
 
Proposition~\ref{prop:invariance} fixes the target of the exact one-step map, but not how hard that map is to learn.
We therefore compare sources through the regression problem that training solves.
At $r=t$ the MeanFlow target \eqref{eq:mftarget} reduces to the conditional velocity $\xbm_1-\xbm_0$, which the network regresses on $(\zbm_t,\ybm)$ with $\zbm_t=(1-t)\xbm_0+t\xbm_1$.
The error left by the best predictor,
\begin{equation}
\mathcal{E}_t(\ybm)\defn\E\!\left[\tr\Cov(\xbm_1-\xbm_0\given\zbm_t,\ybm)\given\ybm\right],
\label{eq:irreducible-error}
\end{equation}
cannot be removed by any network, so a source with a smaller $\mathcal{E}_t$ poses an easier regression problem at time $t$.
 
\begin{lemma}
\label{lem:irreducible}
Let $\sn>0$, fix $\ybm$ and $t\in(0,1)$, and let $\xbm_1\sim p(\cdot\given\ybm)$ have finite second moments.
Let $\xbm_0\sim q_\tau(\cdot\given\ybm)$, and let $\widetilde\xbm_0=\widetilde{\bm K}\ybm+\widetilde{\bm\xi}$ be a source of the form \eqref{eq:source-family} with $\widetilde{\bm\xi}\sim\Ncal(\zerobm,\widetilde{\bm\Sigma})$ and $\widetilde{\bm\Sigma}\succeq\tau^2\Wbm$, each conditionally independent of $\xbm_1$ given $\ybm$.
Let $\mathcal{E}_t(\ybm)$ and $\widetilde{\mathcal{E}}_t(\ybm)$ be the errors \eqref{eq:irreducible-error} for the interpolants $\zbm_t=(1-t)\xbm_0+t\xbm_1$ and $\widetilde\zbm_t=(1-t)\widetilde\xbm_0+t\xbm_1$.
Then $\mathcal{E}_t(\ybm)\le\widetilde{\mathcal{E}}_t(\ybm)$.
In particular, this holds for the isotropic source $\Ncal(\zerobm,\sigma^2\Ibm)$ for every $\sigma\ge\tau$, since $\tau^2\Wbm\preceq\tau^2\Ibm\preceq\sigma^2\Ibm$.
\end{lemma}
 
The lemma requires nothing of the posterior beyond finite second moments, and the center of the source does not enter it: given $\ybm$, the center only shifts $\zbm_t$ by a constant, which conditioning removes.
The center is what Proposition~\ref{prop:optimal-source} chooses to shorten the displacement; the lemma shows that the covariance $\tau^2\Wbm$ is what lowers the regression error.
Both isotropic sources of Table~\ref{tab:ablsource} satisfy its assumption, since $\tau=0.15\le1$ there.
 
\begin{proof}
Given $\ybm$, write $\xbm_0=\myv+\bm\eta$ and $\widetilde\xbm_0=\bm c+\widetilde{\bm\xi}$ with $\bm c\defn\widetilde{\bm K}\ybm$, where $\bm\eta\sim\Ncal(\zerobm,\tau^2\Wbm)$ and $\widetilde{\bm\xi}$ are independent of $\xbm_1$.
Since $\myv$ and $\bm c$ are constants given $\ybm$, the maps $\zbm\mapsto(\zbm-(1-t)\myv)/t$ and $\zbm\mapsto(\zbm-(1-t)\bm c)/t$ are invertible, so conditioning on $\zbm_t$ and on $\widetilde\zbm_t$ is the same as conditioning on
\begin{equation*}
\bar\zbm_t\defn\xbm_1+\tfrac{1-t}{t}\,\bm\eta
\qquad\text{and}\qquad
\bar{\widetilde\zbm}_t\defn\xbm_1+\tfrac{1-t}{t}\,\widetilde{\bm\xi},
\end{equation*}
respectively.
Each is $\xbm_1$ observed through additive Gaussian noise, and the two differ only in that noise.
 
Since $\widetilde{\bm\Sigma}-\tau^2\Wbm\succeq\zerobm$, let $\bm\zeta\sim\Ncal(\zerobm,\widetilde{\bm\Sigma}-\tau^2\Wbm)$ be independent of $(\xbm_1,\bm\eta)$ given $\ybm$.
Then $\bm\eta+\bm\zeta\sim\Ncal(\zerobm,\widetilde{\bm\Sigma})$, and because $\widetilde{\mathcal{E}}_t(\ybm)$ depends only on the joint law of $\xbm_1$ and $\widetilde\zbm_t$ given $\ybm$, we may take $\widetilde{\bm\xi}=\bm\eta+\bm\zeta$.
This gives $\bar{\widetilde\zbm}_t=\bar\zbm_t+\tfrac{1-t}{t}\,\bm\zeta$, so $\xbm_1\to\bar\zbm_t\to\bar{\widetilde\zbm}_t$ is a Markov chain given $\ybm$.
 
Let $\bm\mu_t\defn\E[\xbm_1\given\bar\zbm_t,\ybm]$ and $\widetilde{\bm\mu}_t\defn\E[\xbm_1\given\bar{\widetilde\zbm}_t,\ybm]$.
The Markov property gives $\E[\xbm_1\given\bar\zbm_t,\bar{\widetilde\zbm}_t,\ybm]=\bm\mu_t$, and the tower rule gives $\widetilde{\bm\mu}_t=\E[\bm\mu_t\given\bar{\widetilde\zbm}_t,\ybm]$.
Hence $\xbm_1-\bm\mu_t$ is orthogonal to $\bm\mu_t-\widetilde{\bm\mu}_t$, and
\begin{equation}
\E\!\left[\norm{\xbm_1-\widetilde{\bm\mu}_t}_2^2\given\ybm\right]
=
\E\!\left[\norm{\xbm_1-\bm\mu_t}_2^2\given\ybm\right]
+
\E\!\left[\norm{\bm\mu_t-\widetilde{\bm\mu}_t}_2^2\given\ybm\right].
\label{eq:irr-pythagoras}
\end{equation}
Finally, $\xbm_1-\xbm_0=(\xbm_1-\zbm_t)/(1-t)$ and $\zbm_t$ is fixed once it is conditioned on, so $\Cov(\xbm_1-\xbm_0\given\zbm_t,\ybm)=(1-t)^{-2}\Cov(\xbm_1\given\zbm_t,\ybm)$ and $\mathcal{E}_t(\ybm)=(1-t)^{-2}\,\E[\norm{\xbm_1-\bm\mu_t}_2^2\given\ybm]$; likewise $\widetilde{\mathcal{E}}_t(\ybm)=(1-t)^{-2}\,\E[\norm{\xbm_1-\widetilde{\bm\mu}_t}_2^2\given\ybm]$.
Dividing \eqref{eq:irr-pythagoras} by $(1-t)^2$ gives
\begin{equation*}
\widetilde{\mathcal{E}}_t(\ybm)-\mathcal{E}_t(\ybm)
=
\frac{1}{(1-t)^2}\,\E\!\left[\norm{\bm\mu_t-\widetilde{\bm\mu}_t}_2^2\given\ybm\right]
\;\ge\;0 .
\end{equation*}
\end{proof}

\paragraph{Displacement.}
The same source also shortens the path relative to the isotropic source.
Let $\widetilde\xbm_0\sim\Ncal(\zerobm,\tau^2\Ibm)$ be independent of $(\xbm_1,\nbm)$, and let $E_i\defn\E[(\vbm_i^\top\xbm_1)^2]$ for a right singular vector $\vbm_i$ of $\Abm$.
For any $\xbm_1$ with finite second moments,
\begin{equation}
\E\big[(\vbm_i^\top(\xbm_1-\widetilde\xbm_0))^2\big]-\E\big[(\vbm_i^\top(\xbm_1-\xbm_0))^2\big]
=(1-w_i^2)E_i+\tau^2(1-w_i)^2\;\ge\;0,
\label{eq:displacement-gap}
\end{equation}
with equality if and only if $s_i=0$.
Indeed, writing $a_i\defn\vbm_i^\top\xbm_1$ and $n_i\defn\ubm_i^\top\nbm$ (with $n_i\defn0$ if $s_i=0$), the Tikhonov center gives $\vbm_i^\top(\xbm_1-\xbm_0)=w_ia_i-\frac{s_i}{s_i^2+\lambda}\,n_i-\vbm_i^\top\bm\eta$, a sum of uncorrelated terms with second moments $w_i^2E_i$, $\tau^2w_i(1-w_i)$ and $\tau^2w_i$, whereas the isotropic source gives $E_i+\tau^2$.
%% ---- end optional ----
 
\paragraph{Evaluation.}
Figure~\ref{fig:irreducible} evaluates $\mathcal{E}_t$ for $q_\tau(\cdot\given\ybm)$ and $\Ncal(\zerobm,\tau^2\Ibm)$ in two settings where it is available exactly.
On the two-dimensional mixture of App.~\ref{app:toy} at $y=2.5$, the conditional covariance is computed by quadrature against the closed-form posterior.
The two sources coincide in the null direction there, so the gap comes entirely from the observed direction.
On CelebA Gaussian deblurring, gap is evaluated under a stationary Gaussian prior with the empirical power spectrum, for which $\Abm$, $\Wbm$ and the posterior covariance are simultaneously diagonal in the Fourier basis.
There the gap is present at every $t$ and peaks at $56\%$, since the graded Fourier weights $w_i$ leave the source something to exploit in every direction.

\begin{figure}[ht]
\centering
\includegraphics[width=0.75\textwidth]{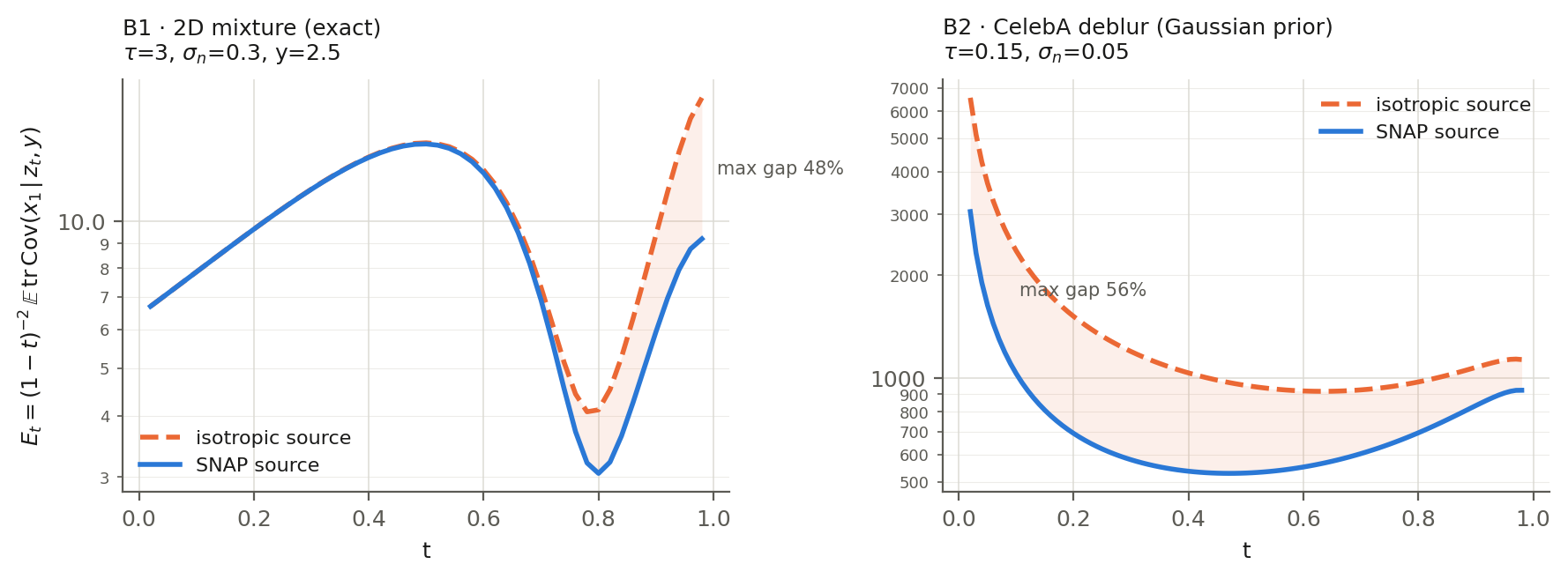}
\caption{\textbf{Irreducible velocity-regression error} $\mathcal{E}_t$ of \eqref{eq:irreducible-error} under the \methodabbrev source $q_\tau(\cdot\given\ybm)$ and the isotropic source $\Ncal(\zerobm,\tau^2\Ibm)$. \textbf{Left:} two-dimensional mixture of App.~\ref{app:toy} ($\tau=3$, $\sn=0.3$, $y=2.5$), exact. \textbf{Right:} CelebA Gaussian deblurring ($\tau=0.15$, $\sn=0.05$) under a stationary Gaussian prior. Lemma~\ref{lem:irreducible} guarantees the ordering; the shaded region is the gap.}
\label{fig:irreducible}
\end{figure}
 
\paragraph{Scope.}
Three limits apply.
First, the lemma concerns the $r=t$ part of the objective. For $r<t$ the target \eqref{eq:mftarget} also contains the network's own derivatives, so the lemma does not transfer verbatim to the interior of the $(r,t)$ range; a fraction $p_{\mathrm{ratio}}$ of training pairs have $r=t$ (Table~\ref{tab:ablratio}).
Second, like the displacement in \eqref{eq:displacement-decomposition}, $\mathcal{E}_t$ decreases as the source covariance shrinks, so it cannot be used to choose $\bm\Sigma$, and the lemma only compares $q_\tau$ with sources that are at least as spread.
A narrower source can have a smaller $\mathcal{E}_t$: on the mixture of App.~\ref{app:toy}, where $\tau=3$, the source $\Ncal(\zerobm,\Ibm)$ has a smaller $\mathcal{E}_t$ than $q_\tau$ at every $t$, because it is narrower in the null direction.
Third, the lemma compares regression problems, not trained samplers; Table~\ref{tab:ablsource} and App.~\ref{app:toy} compare the samplers directly, with everything but the source held fixed.

App.~\ref{app:irreducible} shows that, for any posterior with finite second moments, the velocity-regression error that no network can remove is smaller under $q_\tau(\cdot\given\ybm)$ than under any isotropic source $\Ncal(\zerobm,\sigma^2\Ibm)$ with $\sigma\ge\tau$ (Lemma~\ref{lem:irreducible}); Table~\ref{tab:ablsource} and App.~\ref{app:toy} measure the effect on the learned map.

% \begin{table}[ht]
% \centering\footnotesize
% \setlength{\tabcolsep}{5pt}
% \renewcommand{\arraystretch}{0.95}
% \caption{\textbf{Transport reduction, Proposition~\ref{prop:transport}.} Left- and right-hand sides of \eqref{eq:transport-total} on CelebA, over $200$ test images.}
% \label{tab:propcheck}
% \begin{tabular}{lccccccc}
% \toprule
% Operator & $\tau$ & $\sn$ & $\lambda$ & $\tr\Wbm/n$ & LHS & RHS & rel.\ diff. \\
% \midrule
% Gaussian deblurring & $0.15$ & $0.05$ & $0.111$ & $0.817$ & $18242.90$ & $18242.62$ & $1.5{\times}10^{-5}$ \\
% Random inpainting   & $0.10$ & $0.01$ & $0.010$ & $0.703$ & $5681.13$  & $5680.71$  & $7.4{\times}10^{-5}$ \\
% Box inpainting      & $0.30$ & $0.05$ & $0.028$ & $0.122$ & $21108.12$ & $21107.12$ & $4.8{\times}10^{-5}$ \\
% Super-resolution    & $0.10$ & $0.05$ & $0.250$ & $0.800$ & $4514.06$  & $4515.29$  & $2.7{\times}10^{-4}$ \\
% \bottomrule
% \end{tabular}
% \end{table}

\subsection{Verification on a Gaussian-mixture inverse problem}
\label{app:toy}

We verify the construction on a problem where the posterior is available in
closed form. With a four-component Gaussian-mixture prior on $\R^{2}$ and
observation of the first coordinate ($\Abm=\ebm_0^{\!\top}$, $\sn=0.3$), the
posterior is itself a Gaussian mixture, and for $y\approx2.5$ it is bimodal over
$\xbm[1]\in\{\pm2.5\}$ with the MMSE estimate falling between the modes. The
setting therefore separates posterior sampling from posterior-mean estimation:
a point estimator lands in the empty valley between the modes, and a
mode-collapsed sampler reports the component variance $\varsigma^{2}=0.12$ in
the null direction rather than the posterior variance $6.36$.

\paragraph{Setup.} The clean data $\xbm_1\in\R^{2}$ is drawn from a
four-component equal-weight Gaussian mixture with isotropic component covariance
$\varsigma^{2}\Ibm$ ($\varsigma=0.35$) and means at $(\pm a,\pm a)$, $a=2.5$.
The forward operator observes only the first coordinate, so $m=1$, $n=2$ and
$\nullsp(\Abm)=\mathrm{span}(\ebm_1)$. The working-prior scale is $\tau=3$.

\paragraph{Closed-form posterior.} For each mixture component $k$,
\begin{equation}
\widetilde{\Sigmabm}_k=\big(\Sigmabm_k^{-1}+\Abm^{\!\top}\Abm/\sn^{2}\big)^{-1},
\quad
\widetilde{\boldsymbol{\mu}}_k=\widetilde{\Sigmabm}_k
\big(\Sigmabm_k^{-1}\boldsymbol{\mu}_k+\Abm^{\!\top}\ybm/\sn^{2}\big),
\end{equation}
and 
\begin{equation}
\widetilde{\pi}_k\propto\pi_k\,
\Ncal\big(\ybm;\Abm\boldsymbol{\mu}_k,\Abm\Sigmabm_k\Abm^{\!\top}+\sn^{2}\big), \qquad p(\xbm_1\given\ybm)=\sum_k\widetilde{\pi}_k
\Ncal(\widetilde{\boldsymbol{\mu}}_k,\widetilde{\Sigmabm}_k).
\end{equation}
All quantities labeled ``posterior'' below are computed from this expression.

\paragraph{Network and optimization.} An $\xbm$-prediction MLP (width $256$, depth $4$, SiLU) with random-Fourier embeddings (dimension $64$) for $r$, $t$ and $\sn$, conditioned on $\ybm$, with
$\utheta{r}{t}=\big(\mathrm{net}(\zbm,r,t,\ybm,\sn)-\zbm\big)/(1-r)$. We use the objective
\eqref{eq:spring-loss} with $p=1$ and $c=10^{-3}$; $t$ is logit-normal ($\mu=0$, $\sigma=1$) and $r=t$ with probability $0.75$, else $r\sim\Ucal(0,t)$. AdamW, learning rate $2\times10^{-4}$, batch size $4096$,
$20{,}000$ steps, EMA decay $0.999$. One-step sampling is
$\xbmhat_1=\mathrm{net}(\xbm_0,r{=}0,t{=}1,\ybm,\sn)$; few-step sampling composes $\ubm_{r,t}$ over a uniform partition of $[0,1]$.

\paragraph{The source is exact.} Drawing $2\times10^{5}$ samples from the perturb-and-solve construction in 
\eqref{eq:rto} at $y=2.5$ and comparing against the closed form reproduces
$\Ncal(\myv,\tau^{2}\Wbm)$ to Monte-Carlo accuracy:
$\norm{\Delta\myv}=3.8\times10^{-3}$ and
$\norm{\Delta\mathrm{cov}}_F=4.8\times10^{-3}$ (Table~\ref{tab:rto}). The
anisotropy predicted by \eqref{eq:source-spectrum} is visible directly: the observed
direction is pinned at variance $0.089$ while the null direction carries the
full working-prior variance $\tau^{2}=9$.

\paragraph{Posterior recovery.} 
Figure~\ref{fig:joint} shows the source, the
closed-form posterior, and $2\times10^{4}$ one-step samples for three
in-distribution measurements. The flow transports the vertically-spread source
onto both posterior modes at $\xbm[1]=\pm2.5$ while pinning $\xbm[0]\approx y$,
and the MMSE estimate falls in the empty valley between them. In the null
direction the one-step sampler produces variance $6.66$ against the closed-form
$6.36$; a mode-collapsed sampler would report $\varsigma^{2}=0.12$.
Table~\ref{tab:posterior} reports mean, per-direction variance and
MMD$^{2}$ at three measurements.

\paragraph{Per-direction calibration.} Projecting onto the right singular
vectors of $\Abm$ (here $\ebm_0$ with $s=1$ and $\ebm_1$ with $s=0$), the flow
assigns almost all of the null direction's freedom ($6.66$ against posterior
$6.36$, source $9.0$) and contracts the well-measured direction from the source
value $0.089$ toward the posterior $0.053$, reaching $0.079$
(Table~\ref{tab:calib}).

\paragraph{Gaussian ground truth.} With a single Gaussian prior matched to the
working prior ($\Sigmabm=\tau^{2}\Ibm$, $\tau=\varsigma=0.35$), the source
\emph{is} the posterior and the optimal map is the identity. Retraining
reproduces the source to $\norm{\Delta\myv}=5.6\times10^{-4}$ and matches the
Gaussian posterior with mean errors $\le0.05$ and
MMD$^{2}\le3\times10^{-3}$ (Table~\ref{tab:gauss}).

\begin{table}[ht]
\centering\small
\caption{Posterior match at three in-distribution measurements.}
\label{tab:posterior}
\begin{tabular}{cccccc}
\toprule
$y$ & mean (post) & mean (\methodabbrev) & $\norm{\Delta\text{mean}}$
    & $\Var$ (observed, null) post $\to$ \methodabbrev & $\mathrm{MMD}^{2}$ \\
\midrule
$2.5$  & $(2.5,0)$   & $(2.474,0.065)$   & $0.069$ & $(0.052,6.37)\!\to\!(0.079,6.66)$ & $4.0{\times}10^{-3}$ \\
$2.0$  & $(2.212,0)$ & $(2.045,0.051)$   & $0.174$ & $(0.052,6.37)\!\to\!(0.084,6.64)$ & $1.1{\times}10^{-2}$ \\
$-2.5$ & $(-2.5,0)$  & $(-2.444,0.108)$  & $0.121$ & $(0.052,6.37)\!\to\!(0.081,6.59)$ & $2.6{\times}10^{-3}$ \\
\bottomrule
\end{tabular}
\end{table}

\begin{table}[ht]
\centering\small
\caption{Single-Gaussian ground-truth check. The map reduces to the Wiener estimate, with the posterior mean recovered to within $0.05$ and the posterior variances over-estimated by $37$--$39\%$.}
\label{tab:gauss}
\begin{tabular}{cccccc}
\toprule
$y$ & mean (post) & mean (\methodabbrev) & $\norm{\Delta\text{mean}}$
    & $\Var$ (observed, null) post $\to$ \methodabbrev & $\mathrm{MMD}^{2}$ \\
\midrule
$0$ & $(0,0)$     & $(-0.040,0.026)$ & $0.048$ & $(0.052,0.123)\!\to\!(0.071,0.171)$ & $2.7{\times}10^{-3}$ \\
$1$ & $(0.576,0)$ & $(0.553,0.012)$  & $0.026$ & $(0.052,0.123)\!\to\!(0.069,0.160)$ & $1.7{\times}10^{-3}$ \\
\bottomrule
\end{tabular}
\end{table}

\begin{table}[t]
\centering\small
\caption{The square-root-free sampler \eqref{eq:rto} reproduces
$\Ncal(\myv,\tau^{2}\Wbm)$ to Monte-Carlo accuracy at $y=2.5$.}
\label{tab:rto}
\begin{tabular}{lccc}
\toprule
 & $\myv$ & $\tau^{2}\Wbm$ (diag) & error \\
\midrule
closed form   & $[2.4752,\ 0]$      & $[0.0891,\ 9.000]$ & --- \\
RTO \eqref{eq:rto} & $[2.4752,\ 0.0038]$ & $[0.0889,\ 8.995]$ &
  $\norm{\Delta\myv}=3.8{\times}10^{-3}$, $\norm{\Delta\mathrm{cov}}_F=4.8{\times}10^{-3}$ \\
\bottomrule
\end{tabular}
\end{table}

\begin{table}[t]
\centering\small
\caption{Per-direction variance at $y=2.5$ (\eqref{eq:source-spectrum}). The
source over-disperses both directions relative to the posterior; the trained
flow moves both toward it, almost exactly in the null direction and partially in
the observed one.}
\label{tab:calib}
\begin{tabular}{lccc}
\toprule
direction & $\Var$ posterior & $\Var$ one-step \methodabbrev & $\Var$ source \\
\midrule
$\ebm_0$ (observed, $s=1$) & $0.0518$ & $0.0791$ & $0.0891$ \\
$\ebm_1$ (null, $s=0$)     & $6.372$  & $6.656$  & $9.000$ \\
\bottomrule
\end{tabular}
\end{table}

\begin{figure*}[t]
    \centering
    \includegraphics[width=0.75\textwidth]{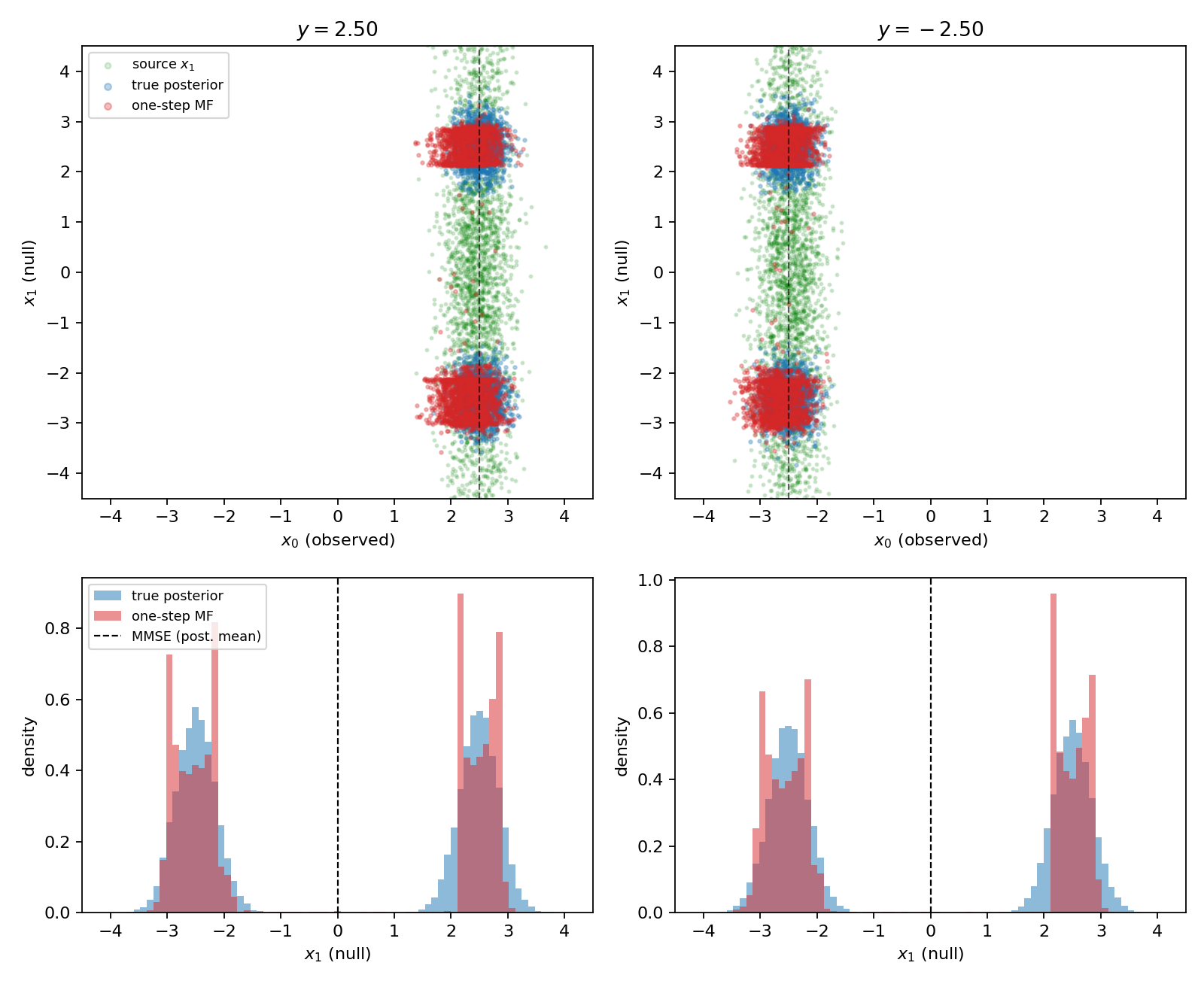}
    \caption{
    Posterior recovery for observations $y=2.5$ (left) and $y=-2.5$ (right).
    The top row compares samples from the source distribution (green), the
    closed-form posterior (blue), and the one-step model (red). Dashed vertical
    lines indicate the observed values. The bottom row shows the corresponding
    marginals along the unobserved coordinate $x_1$; the dashed line marks the
    posterior mean. The one-step model captures both posterior modes, although
    it exhibits some within-mode distortion relative to the exact posterior.
    }
    \label{fig:joint}
\end{figure*}

\section{Implementation Details}
\label{app:implement_details}

\subsection{Datasets and evaluation protocol}
\label{app:datasets}

\paragraph{CelebA.} Official aligned images with the canonical partition of $162{,}770/19{,}867/19{,}962$ for the train/validation/test images, center-cropped to $178{\times}178$ and resized to $128{\times}128$. We report results on the first $100$ images of the test set. 

\paragraph{AFHQ-Cat.} $5{,}153/500/100$ train/validation/test images, resized to $256{\times}256$. AFHQ-Cat has no official validation split. We followed~\cite{martin2025pnpflow} for train/val/test split.

\paragraph{fastMRI brain.} The AXT2 subset, complex multi-coil $k$-space at $320{\times}320$. We split at the volume (patient) level into $615/77/77$ train/val/test volumes, giving $4{,}305/539/539$ slices after discarding the first four and last five slices of every volume, which contain little or no anatomy. Coil handling is full multi-coil SENSE, not RSS or an emulated single-coil reduction. The ground truth is the ESPIRiT coil-combined complex image $\xbm_1 = \sum_i \bar{\Sbm}_i\,\mathrm{crop}(\Fbm^H \kbm_i)$ on a $320{\times}320$
grid, retained as two real channels (real/imaginary) and normalized per slice by $\max|\xbm_1|$; phase is modeled throughout and magnitude is taken only when computing PSNR/SSIM. The forward operator is $\Abm = \Mbm\Fbm\Sbm$ with the same per-slice ESPIRiT maps at train and test time, zero-padded to $20$ coils.
Because the coils couple, the measurement-space Gram $\Abm\Abm^H$ is not diagonal and the posterior-source solve admits no closed form, unlike the inpainting, super-resolution, and deblurring operators, where $\Abm^H\Abm$ is diagonal in the pixel or Fourier basis and the solve costs $O(n)$ or a single FFT pair. We therefore solve it with conjugate gradient in the algebraically equivalent image-space form
$\Abm^H(\Abm\Abm^H+\lambda\Ibm)^{-1} = (\Abm^H\Abm+\lambda\Ibm)^{-1}\Abm^H$, whose spectrum lies in $[\lambda,\,{\approx}1]$ and whose right-hand side already
lies in $\range(\Abm^H)$; the measurement-space form is rank-deficient by the coil count and, at the same iteration budget, returns an anchor worse than the plain adjoint.

\subsection{Solutions of the source construction}
\label{app:sourcesolve}

The source \eqref{eq:rto} requires one application of
$\Kbm = \Abm^{\!\top}(\Abm\Abm^{\!\top}+\lambda\Ibm)^{-1}$ per measurement, with
$\lambda=\sn^{2}/\tau^{2}$. Writing $\Abm=\Ubm\Sigmabm\Vbm^{\!\top}$, the
covariance is $\tau^{2}\Wbm$ with $\Wbm=\Vbm\diag(w_i)\Vbm^{\!\top}$ and
$w_i=\lambda/(s_i^{2}+\lambda)$. We summarise the cases used in this work.

\textbf{Image deblurring.}
For deblurring $\Abm=\Hbm$ is a spatially invariant blur. Under circular boundary conditions $\Hbm=\Fbm^{\!H}\bm \Lambda\Fbm$ with $\bm\Lambda$ the Fourier coefficients $\hat h$ of the kernel, so
$\Hbm\Hbm^{\!H}=\Fbm^{\!H}\abs{\bm\Lambda}^{2}\Fbm$ and
\begin{equation}
\myv = \Fbm^{\!H}\!\left[\frac{\overline{\bm\Lambda}\,\Fbm\ybm}
{\abs{\bm\Lambda}^{2}+\lambda}\right],
\qquad w_i = \frac{\lambda}{\abs{\hat h_i}^{2}+\lambda},
\label{eq:sourcedeblur}
\end{equation}
elementwise in the Fourier domain, at the cost of two FFTs. Because
$\abs{\hat h_i}>0$ for every $i$, no $w_i$ reaches $1$ and no direction is unmeasured; the spectrum instead decays smoothly.

\textbf{Image inpainting.}
For inpainting $\Abm=\Mbm$ is a row-selection matrix that keeps the observed coordinates, so $\Mbm\Mbm^{\!\top}=\Ibm_m$ and the solve is elementwise:
\begin{equation}
(\myv)_i =
\begin{cases}
\dfrac{y_i}{1+\lambda}, & i\in\Omega,\\[6pt]
0, & i\notin\Omega,
\end{cases}
\qquad
w_i =
\begin{cases}
\dfrac{\lambda}{1+\lambda}, & i\in\Omega,\\[6pt]
1, & i\notin\Omega,
\end{cases}
\label{eq:sourceinpaint}
\end{equation}
where $\Omega$ is the set of observed pixels. Unobserved coordinates carry the full working-prior variance $\tau^{2}$, while observed coordinates are shrunk toward the prior mean by the factor $1/(1+\lambda)$ and retain variance $\tau^{2}\lambda/(1+\lambda)$. Box and random inpainting share this
structure and differ only in the geometry of $\Omega$.

\textbf{Image super-resolution.}
We use box-decimated super-resolution by implementing strided subsampling with no anti-alias prefilter,
$\Abm\xbm=\xbm[\ldots,::s,::s]$, so $\Abm\Abm^{\!\top}=\Ibm$, $\Abm^{\!\top}$ is
zero-filled upsampling, and \eqref{eq:sourceinpaint} applies with $\Omega$ the sampled grid locations. 

\textbf{Multi-coil MRI.}
With $\Abm=\Mbm\Fbm\Sbm$ the sensitivity maps $\Sbm$ are not unitary, so
$\Abm\Abm^{\!\top}$ is not diagonal and the solve with $\Abm\Abm^{\!\top}+\lambda\Ibm$ has no closed form. We solve it with conjugate gradients, warm-started at the zero-filled adjoint. Measured iteration counts and cost are reported in Table~\ref{tab:cost}.

\begin{table}[t]
\centering
\small
\caption{Cost of the source construction \eqref{eq:rto}. Larger $\lambda$
(noisier data) improves conditioning and accelerates CG.}
\label{tab:cost}
\begin{tabular}{llll}
\toprule
Problem & $\Abm\Abm^{\!\top}$ structure & Cost & Notes \\
\midrule
Inpainting (random)       & $\Ibm_m$             & free            & $\Wbm$ diagonal \\
Inpainting (box)          & $\Ibm_m$             & free            & identical structure \\
MRI, single-coil          & $\Ibm_m$             & free            & diagonal in $k$-space \\
SR, average-pool          & $d^{-2}\Ibm$         & free            & disjoint rows of norm $1/d$ \\
Deblurring, Gaussian      & diag.\ in Fourier    & 2 FFTs          & $\Abm\Abm^{\!\top}=\abs{\hat h}^{2}$ \\
Deblurring, uniform       & diag.\ in Fourier    & 2 FFTs          & $\hat h$ has exact zeros \\
SR, bicubic               & block-Fourier        & $O(n\log n)$    & polyphase folding identity \\
MRI, multi-coil           & not diagonal         & CG, $10$--$20$ it.\ typical  & standard SENSE \\
\bottomrule
\end{tabular}
\end{table}

\textbf{General linear operators.}
For a general $\Abm$ with no exploitable structure, \eqref{eq:rto} is solved by
conjugate gradients. All operators used in this paper fall under the cases
above. The solve is required once per draw and does not depend on $t$, $r$, or the network.
% , and for a fixed operator the structures above reduce it to a diagonal inverse or a short conjugate-gradient run.
For a fixed operator, the solve can be reduced to a diagonal inverse for diagonal (and Fourier diagonal) forward models or use a short conjugate gradient run for most others.

\subsection{Training and implementation details}
\label{app:train_imp}

\paragraph{Architecture.}
The backbone is a DDPM-style residual U-Net (\texttt{GroupNorm}$_{32}$, Swish,
variance-scaled initialization, zero-initialized output and residual-branch
convolutions) with 
% two additional 
three total
scalar conditioning inputs. The time $t$ is embedded by a sinusoidal positional encoding followed by a two-layer MLP into a
$4\,\mathrm{ch}$-dimensional vector; the second MeanFlow time $r$ and the noise
level $\sn$ each get their own identical embedding, and the three are
\emph{summed} before entering the body, where each residual block adds a
per-channel affine projection of the sum to its first activation. The $\sn$
embedding takes $10\,\sn$ as input so that the small noise levels used here
occupy a useful part of the sinusoidal range. Conditioning on $\ybm$ is by
\emph{concatenation} rather than FiLM, with input
$[\,\zbm_t,\ \mathrm{anchor}]$ of $2C$ channels, where
the anchor is  the adjoint $\Abm^{\!\top}\ybm$.

Configurations: CelebA $128{\times}128$ uses $\mathrm{ch}=64$,
$\mathrm{ch\_mult}=(1,2,2,2)$, two residual blocks per level and attention at
$16^2$ ($9.1$M parameters); AFHQ-Cat $256{\times}256$ uses $\mathrm{ch}=64$,
$\mathrm{ch\_mult}=(1,2,4,4,8,8)$ down to an $8{\times}8$ bottleneck, two
residual blocks, attention at $32^2$ and $16^2$ and dropout $0.1$ ($110.5$M);
fastMRI brain $320{\times}320$ uses $\mathrm{ch}=32$,
$\mathrm{ch\_mult}=(1,2,4,8,8)$, four residual blocks and attention at $40^2$
($40.7$M).
Attention is a hand-written block of $1{\times}1$ convolutional $Q,K,V$
projections, a batched matrix product scaled by $C^{-1/2}$, and a
\texttt{float32} softmax, since \texttt{torch.func.jvp} has no forward-mode rule for the fused scaled-dot-product kernels. The same restriction rules out \texttt{torch.utils.checkpoint}, so activation memory is a hard ceiling and width is added only at the two deepest levels.

\begin{algorithm}[t]
\caption{ \methodabbrev training step.}
\label{alg:train}
\begin{algorithmic}[1]
\State $\xbm_1\sim p$, \quad $\ybm \gets \Abm\xbm_1+\nbm$
   \Comment{training pair, \eqref{eq:fwd}}
\State $t,r \gets \textsc{SampleTimes}()$
   \Comment{$t$ logit-normal; $r=t$ with probability $p_{\mathrm{ratio}}$}
\State $\xbm_0 \gets \textsc{SampleSource}(\ybm,\sn)$
   \Comment{perturb-and-solve, \eqref{eq:rto}}
\State $\zbm_r \gets (1-r)\xbm_0 + r\,\xbm_1$; \quad $\vbm \gets \xbm_1-\xbm_0$
\State $\ubm,\ \dd_r\ubm \gets \textsc{Jvp}\big(\ubm^{\thetabm};\,(\zbm_r,r,t),\,(\vbm,1,0)\big)$
   \Comment{one forward-mode pass}
\State $\ubm_{\mathrm{tgt}} \gets \vbm + (t-r)\,\sg{\dd_r\ubm}$
   \Comment{MeanFlow target, \eqref{eq:mftarget}}
\State $\ell \gets \norm{\ubm-\ubm_{\mathrm{tgt}}}_2^{2}/n$
\State $\Lcal \gets \sg{\omega(\ell)}\,\ell$
   \Comment{weighted objective, \eqref{eq:spring-loss}}
\State $\thetabm \gets \textsc{AdamW}\big(\thetabm,\nabla_{\thetabm}\Lcal\big)$
   \Comment{optimizer step}
\end{algorithmic}
\end{algorithm}

\paragraph{Training.}
Given a clean image $\xbm_1$ we form $\ybm=\Abm\xbm_1+\nbm$, draw the source
$\xbm_0\sim\Ncal(\myv,\tau^2\Wbm)$ by \eqref{eq:rto}, and set
$\zbm_t=(1-t)\xbm_0+t\xbm_1$ and $\vbm=\xbm_1-\xbm_0$. Times are drawn as
$t=\sigma(\mu+s\,\epsilon)$ with $\epsilon\sim\Ncal(0,1)$, $\mu=0$ and $s=1$
(logit-normal on $(0,1)$); with probability $p_{\mathrm{ratio}}=0.5$ we set
$r=t$, and otherwise $r\sim\Ucal(0,t)$. A further $10\%$ of the $r\neq t$ samples, or $5\%$ of the batch, is redrawn at the corner $r\sim\Ucal(0,0.1)$, $t\sim\Ucal(0.9,1)$, because the base schedule reaches the slice at which one-step sampling is scored with probability only $7.3\times10^{-4}$.

The network is an $\xbm_1$-predictor, so
$\ubm(\zbm_r,r,1)=\big(f_{\thetabm}(\zbm_r,r,1)-\zbm_r\big)/(1-r)$. The total
derivative is one forward-mode pass, with the differentiated arguments
$(\zbm,r,t)$ carrying tangent $\vbm$ on $\zbm$, $1$ on $r$ and $0$ on $t$,
while the anchor, the coverage map and $\sn$ are closed over and carry zero
tangent. The regression target is
$\ubm_{\mathrm{tgt}}=\vbm+(t-r)\,\sg{\dd\ubm/\dd r}$ and the loss is \eqref{eq:spring-loss} with $\bm{\Delta}=\ubm^{\thetabm}-\ubm_{\mathrm{tgt}}$, $p=1$ and $c=10^{-3}$. The stop-gradient is on the
weight only, and a second stop-gradient is applied to the derivative term. On CelebA and AFHQ, $\ell$ is first divided by its all-reduced running mean (decay
$0.99$) so that $c$ is dimensionless and the loss shape is stable across $\tau$
and across training; the MRI runs use the unnormalized form. No endpoint
$\norm{\hat\xbm_1-\xbm_1}_2^2$ term is used, since its population minimizer is
the posterior mean and would invalidate every calibration metric.

Optimization is AdamW with weight decay
$10^{-4}$, cosine annealing to $10^{-6}$ over the full epoch budget, no warm-up
and gradient-norm clipping at $1.0$. The learning rate is $2\times10^{-4}$ on
CelebA and AFHQ and $1\times10^{-4}$ on MRI. Everything is trained in
\texttt{float32}, since forward-mode AD is not autocast-safe. The effective
batch is $32$ throughout (CelebA $8\times4$ GPUs; AFHQ $2\times4$ accumulation
steps $\times4$ GPUs; MRI $4\times4$ GPUs). Budgets are $200$ epochs ($50{,}000$ steps) on CelebA, $120$ epochs
($15{,}000$ steps) for AFHQ super-resolution and $80$ epochs ($10{,}000$ steps)
for the remaining AFHQ operators, and $80$ epochs ($43{,}040$ steps) for brain
MRI, where an epoch is a fixed-size draw with replacement rather than a pass
over the data. EMA is implemented but disabled in all reported runs, having
never beaten the raw weights.

\paragraph{Choice of $\tau$.}
We set $\tau^2$ to the mean posterior variance left along the directions the measurement leaves open, estimated offline on $256$ training images as $\tau^2=\langle\bm{d},\Wbm\bm{d}\rangle/\operatorname{tr}\Wbm$, where $\bm{d}=\xbm-\hat\xbm(\ybm)$ is the residual of the linear-MMSE (Wiener) estimate under the empirical image power spectrum.
Since $\Wbm$ depends on $\tau$, we solve this by fixed-point iteration.
The linear estimate is suboptimal, so its residual upper-bounds the posterior variance and $\tau$ errs toward over-dispersion.
This gives $\tau=0.1$ for $\times2$ super-resolution and random inpainting on CelebA, $0.15$ for deblurring and $0.3$ for box inpainting; $\tau=0.1$ for $\times4$ super-resolution, random inpainting and deblurring on AFHQ and $0.5$ for box inpainting; and $\tau=0.25$ for brain MRI.
The ablation in App.~\ref{app:ablations} shows that reconstruction quality changes little across this range. Dataset splits and
preprocessing are given in App.~\ref{app:datasets}.

\paragraph{Source solves.}
The draw \eqref{eq:rto} needs only $\Abm$, $\Abm^{\!\top}$, and one solve with
$(\Abm\Abm^{\!\top}+\lambda\Ibm)^{-1}$, where $\lambda=\sn^2/\tau^2$ is floored at
$10^{-8}$, and for the operators used here that solve is a diagonal inverse, except for multi-coil MRI.
% is available in closed form, except for multi-coil MRI. 

Circular Gaussian deblurring is Fourier-diagonal, with $\Abm\Abm^{\!\top}=|H(\omega)|^2$
and one division per mode. 
Box-decimation super-resolution has
$\Abm\Abm^{\!\top}=\Ibm$ on the low-resolution grid, so the solve is the scalar
division by $1+\lambda$. 
Inpainting has $\Abm\Abm^{\!\top}=\Mbm$ for a $0/1$
mask, giving a pixel-diagonal division by $\Mbm+\lambda$. 
% All three are $O(n)$ or a single FFT pair and add nothing measurable to a training step.
Super-resolution and inpainting are $\mathcal{O}(n)$ and Gaussian deblurring is $\mathcal{O}(n \log n)$ from the FFT pair; none add any measurable time to a training step.

Multi-coil MRI is the only iterative case, and we solve it in image space using
$\Abm^{\!\top}(\Abm\Abm^{\!\top}+\lambda\Ibm)^{-1}=(\Abm^{\!\top}\Abm+\lambda\Ibm)^{-1}\Abm^{\!\top}$,
running conjugate gradient on $\Abm^{\!\top}\Abm+\lambda\Ibm$, whose spectrum
lies in $[\lambda,\approx1]$ and whose right-hand side already lies in
$\mathrm{range}(\Abm^{\!\top})$, with a cap of $200$ iterations and relative-residual tolerance $10^{-6}$, which is typically reached in $10$--$20$ iterations and per-sample reductions so that batched slices do not
share a step size. The measurement-space route is implemented but not used,
since $\Abm\Abm^{\!\top}$ acts on $k$-space while factoring through the image
and is therefore rank-deficient by the coil count on top of the mask; a short CG
there leaves the anchor over-scaled at $6.5$\,dB, worse than the plain adjoint
at $24.4$\,dB, against $26.0$\,dB for the image-space form at equal cost.
Requesting the anchor $\myv$ inside the sampler splits the single solve into two
by linearity, which is free for the diagonal operators and doubles the CG only
for MRI.

\paragraph{Cost.}
All models were trained on four A6000 GPUs with manual data-parallel gradient
all-reduce. Measured end-to-end wall-clock: AFHQ-Cat super-resolution
$6$\,h\,$07$\,m for $120$ epochs; AFHQ random inpainting, box inpainting and
deblurring $4$\,h\,$04$\,m, $4$\,h\,$00$\,m and $4$\,h\,$02$\,m for $80$ epochs
each; CelebA box inpainting $3$\,h\,$29$\,m for $200$ epochs; and brain MRI at
$\times8$ acceleration $19$\,h\,$05$\,m for $80$ epochs. Throughput on the
$256$-pixel AFHQ configuration is $6.6$ images/s/GPU at $26.1$\,GB peak memory
with batch size $2$; the $110.5$M model costs $18\%$ more memory and $21\%$ less
throughput than a $28.0$M baseline, because memory is activation- rather than
parameter-dominated and the added width sits at the $16^2$ and $8^2$ levels. 
At inference an $M$-sample posterior costs $M$ network evaluations and $M$ source draws.
Each draw is $O(n)$ for super-resolution and inpainting, $O(n \log n)$ for Gaussian deblurring, and a CG solve for multi-coil MRI.

\paragraph{Baseline budgets.}
The NFE column of each table counts network evaluations per reported output.
\methodabbrev and the MSE regressor use one evaluation per draw, so an $M$-draw row costs $M$ evaluations.
On natural images PnP-GS uses $23$, DDRM $20$, DiffPIR $100$, OT-ODE $180$ and DPS $1000$; on fastMRI, DiffPIR uses $100$ and CSGM, DPS, DAPS and PnP-DM use $1000$.
The subscript on PnP-Flow1, PnP-Flow5, Flower1-OT and Flower5-OT is $N_{\mathrm{Avg}}$, the number of network evaluations averaged at each time step, so a run of $N$ steps costs $N\cdot N_{\mathrm{Avg}}$ evaluations; we use the values tabulated by \cite{pourya2026flower}.
On CelebA both methods use $N=100$ for every operator, so Tab.~\ref{tab:repro_celeba} reports the five-evaluation variants at $500$ evaluations, and the single-evaluation variant Flower1-OT of Tab.~\ref{tab:fid_celeba} at $100$.
On AFHQ-Cat the step counts depend on the operator: PnP-Flow uses $N=500$ for deblurring and super-resolution, $200$ for random inpainting and $100$ for box inpainting, while Flower uses $N=100$, $500$, $200$ and $100$ on those same four operators.
The budgets behind the NFE range quoted in Tab.~\ref{tab:repro_afhq} are therefore $2500$, $2500$, $1000$ and $500$ evaluations for PnP-Flow5, and $500$, $2500$, $1000$ and $500$ for Flower5-OT.
Every \methodabbrev entry in Tab.~\ref{tab:fid_celeba} is a single draw ($M{=}1$); averaged draws are reported only in the distortion tables and in the ablations of App.~\ref{app:ablations}.

\subsection{Metric details}
\label{app:metrics}

For natural images, PSNR is computed on $[0,1]$ with data range $1$, SSIM on the clamped $[0,1]$ image, and LPIPS with an AlexNet backbone on the clamped $[-1,1]$ image. Images are normalized to $[-1,1]$ for training and mapped back to $[0,1]$ for scoring.
For MRI, PSNR and SSIM are computed on magnitude images, per slice, with the data range set to the maximum intensity of the ground-truth slice.
We do not report LPIPS on MRI: its backbone is trained on RGB natural images and its perceptual judgments do not transfer to grayscale magnitude reconstructions.

\paragraph{Calibration ratio.}
We check whether the spread of a sampler's draws has the right size for a sampler of $p(\xbm\mid\ybm)$.
Let $p_{\mathrm{s}}(\xbm\mid\ybm)$ be the distribution of the draws and $\bm{\mu}_{\mathrm{s}}(\ybm)$ its mean, and let $(\xbm_1,\ybm)$ be a ground-truth image and its measurement.
Define
\begin{align*}
B &\defn \E_{\xbm_1,\ybm}\big[\norm{\bm{\mu}_{\mathrm{s}}(\ybm)-\xbm_1}_2^2\big],
&
V &\defn \E_{\ybm}\big[\operatorname{tr}\operatorname{Cov}_{\mathrm{s}}(\xbm\mid\ybm)\big].
\end{align*}
If $p_{\mathrm{s}}=p$, then $\bm{\mu}_{\mathrm{s}}(\ybm)=\E[\xbm\mid\ybm]$ and
\begin{equation*}
B=\E_{\ybm}\Big[\E_{\xbm_1\mid\ybm}\big[\norm{\E[\xbm\mid\ybm]-\xbm_1}_2^2\big]\Big]=\E_{\ybm}\big[\operatorname{tr}\operatorname{Cov}(\xbm\mid\ybm)\big]=V.
\end{equation*}
Hence $V/B=1$ is a necessary condition for $p_{\mathrm{s}}=p$, and we report the calibration ratio
\begin{equation}
\Delta_{\mathrm{cal}}\defn\frac{V}{B}.
\label{eq:calib-ratio}
\end{equation}
Values below $1$ indicate draws with too little spread, collapsed toward a point estimate, and values above $1$ indicate over-dispersed draws.

\textbf{Estimation.}
For each test pair $(\xbm_1^{(i)},\ybm^{(i)})$, $i=1,\dots,N$, we draw $\sbm_{(1)}(\ybm^{(i)}),\dots,\sbm_{(M)}(\ybm^{(i)})$ with mean $\bar\sbm(\ybm^{(i)})$ and use
\begin{align*}
\hat V &\defn \frac{1}{N(M-1)}\sum_{i=1}^N\sum_{j=1}^M\norm{\sbm_{(j)}(\ybm^{(i)})-\bar\sbm(\ybm^{(i)})}_2^2,
\\
\hat B &\defn \frac{1}{N}\sum_{i=1}^N\norm{\bar\sbm(\ybm^{(i)})-\xbm_1^{(i)}}_2^2-\frac{\hat V}{M},
\end{align*}
which are unbiased for $V$ and $B$; the correction $\hat V/M$ removes the variance that remains in the mean of $M$ draws.
We use $M=100$.
$\Delta_{\mathrm{cal}}$ is related to the spread-skill ratio used in ensemble forecasting, which compares the standard deviation of the ensemble with the root-mean-squared error of its mean and therefore corresponds to $\sqrt{V/B}$.

\paragraph{FID.}
FID is computed with InceptionV3 pool3 features on 2048 single-draw reconstructions, quantised to uint8, against reference statistics from 2048 ground-truth images.

\end{document}